\documentclass[10pt,a4paper,twocolumn]{sg_preprint}

\usepackage[authoryear,round]{natbib}
\usepackage{mathtools,bm}
\usepackage{array}
\usepackage{placeins}
\usepackage{microtype}
\hypersetup{pdftitle={Shared Gaussianization: What Gaussian Regularizers Certify About Contrastive Learning, and What They Miss},pdfauthor={Ruoyu Zhao, Yuting Chen, Jinheng Zhang, Zhehao Zou, Tong Che},colorlinks=true,linkcolor=blue!50!black,citecolor=blue!50!black,urlcolor=blue!50!black}
\usepackage[capitalise,nameinlink]{cleveref}

\newcommand{\R}{\mathbb{R}}
\newcommand{\E}{\mathbb{E}}
\newcommand{\Sd}{\mathbb{S}^{d-1}}
\newcommand{\J}{\mathcal{J}}
\newcommand{\Dlift}{D_{\rm lift}}
\newcommand{\Db}{D_\beta}
\newcommand{\Ub}{\mathcal{U}_\beta}
\newcommand{\Lb}{\mathcal{L}_\beta}
\newcommand{\dd}{\mathrm{d}}
\newcommand{\hypF}{{}_1F_1}

\theoremstyle{plain}
\newtheorem{theorem}{Theorem}
\newtheorem{proposition}{Proposition}
\newtheorem{lemma}{Lemma}
\newtheorem{corollary}{Corollary}
\theoremstyle{definition}

\newtheorem{remark}{Remark}
\crefname{lemma}{Lemma}{Lemmas}
\crefname{proposition}{Proposition}{Propositions}
\crefname{remark}{Remark}{Remarks}
\crefname{corollary}{Corollary}{Corollaries}
\crefname{appendix}{Appendix}{Appendices}\Crefname{appendix}{Appendix}{Appendices}
\crefname{subappendix}{Appendix}{Appendices}\Crefname{subappendix}{Appendix}{Appendices}
\crefname{subsubappendix}{Appendix}{Appendices}\Crefname{subsubappendix}{Appendix}{Appendices}

\title{Shared Gaussianization: What Gaussian Regularizers Certify About Contrastive Learning, and What They Miss}
\runningtitle{Shared Gaussianization: What Gaussian Regularizers Certify About Contrastive Learning, and What They Miss}
\preprintlabel{Preprint}

\author[1,$\ast$]{Ruoyu Zhao}
\author[2,$\ast$]{Yuting Chen}
\author[3]{Jinheng Zhang}
\author[4]{Zhehao Zou}
\author[5,$\dagger$]{Tong Che}
\affil[1]{City University of Hong Kong, \texttt{ruoyuzhao8-c@my.cityu.edu.hk}}
\affil[2]{Georgia Institute of Technology, \texttt{yuting3123@gmail.com}}
\affil[3]{University of Pennsylvania, \texttt{jinhengz@seas.upenn.edu}}
\affil[4]{The Chinese University of Hong Kong, \texttt{1155257174@link.cuhk.edu.hk}}
\affil[5]{NVIDIA Research, \texttt{tongc@nvidia.com}}
\firstpagenote{$^\ast$Equal contribution.\quad $^\dagger$Project lead.}

\begin{document}

\begin{abstract}
What can a distribution-matching regularizer such as SIGReg in LeJEPA certify about contrastive learning? We study shared Gaussianization (SG), a characteristic-function Gaussianity test on the average of two normalized views, scaled by an independent $\chi_d$ radius. Because disagreeing views shorten the average, one test detects both misalignment and non-uniformity. SG vanishes exactly at the aligned, uniform minimizers of population InfoNCE, and under equal marginals it bounds the InfoNCE excess by $4\cdot3^{3/4}\beta\sqrt{\mathcal J}$ plus an $O(\mathcal J)$ term. The rate and the dimension-free leading constant are sharp, and no squared mean-embedding distance on view pairs achieves a faster rate. With an explicit alignment term, a rotation-invariant uniformity test gives a linear bound if and only if its spectrum dominates that of $e^{\beta u^\top v}$; SG's own test does, Gaussian kernels $e^{-\gamma\|u-v\|^2}$ qualify exactly when $\gamma\ge\beta/2$, and moment matching never does. Away from the optimum, the objectives differ. Along an isotropic nuisance channel, pure SG lowers its loss by adding per-view nuisance whenever the shared code is non-uniform. An alignment weight above the channel's gain makes the nuisance-free solution a strict local minimizer; for LeJEPA, the same rule gives a critical SIGReg weight that decreases with the batch size. At finite batch size, an off-diagonal U-statistic removes a plug-in bias toward misalignment. In controlled latent-variable models, pure SG retains per-view style, an alignment weight above the measured gain removes it, and for LeJEPA at three batch sizes the measured gain separates the encoders that retain style from those that do not. InfoNCE training also reaches a lower SG$_{0.2}$ loss than SG$_{0.2}$ training from scratch, which points to an optimization gap.
\end{abstract}

\maketitle

\section{Introduction}

Self-supervised representation learning commonly uses two families of joint-embedding objectives: contrastive methods and distribution-matching methods. \emph{Contrastive} methods such as InfoNCE and SimCLR \citep{oord2018cpc,chen2020simclr} pull positive pairs together and push the rest of the batch away; population InfoNCE is \emph{alignment plus uniformity} on the sphere \citep{wang2020understanding}. \emph{Distribution-matching} methods drop the negatives and prevent collapse with a test of the embedding distribution: variance--covariance penalties \citep{bardes2022vicreg,zbontar2021barlow}, spherical uniformity tests \citep{zheng2023sfrik,nicollier2026sphere}, and the sketched isotropic Gaussian regularizer (SIGReg) of LeJEPA \citep{balestriero2025lejepa}, recently recast as a kernel maximum mean discrepancy (MMD) \citep{zimmermann2025kerjepa}. Such a test needs no pairwise similarity matrix; LeJEPA enforces view agreement with a separate alignment (prediction) term. Tuned carefully, the families achieve comparable accuracy \citep{garrido2023duality}, which leaves open what a distribution test certifies about contrastive learning and whether the two learn the same representation.

We ask whether a single distribution-matching loss can enforce both alignment and uniformity, and study \emph{shared Gaussianization} (SG). For normalized embeddings $U,V\in\Sd$ of two augmented views, SG tests the law of $Z=RC$, with centroid $C=\frac12(U+V)$ and an independent radius $R\sim\chi_d$, against $\mathcal N(0,I_d)$ using SIGReg's characteristic-function loss $\J$. A uniform direction times an independent $\chi_d$ radius is exactly Gaussian, while $\|C\|^2=1-\frac14\|U-V\|^2$, so disagreement shortens the centroid, and the same test sees both global non-uniformity and pairwise misalignment.

To compare SG with population InfoNCE, we separate three questions: whether two objectives share the same ideal solution, whether one quantitatively controls the other, and whether they favor the same representations away from that solution. SG answers yes to the first two and no to the third. Away from the optimum, SG can keep per-view nuisance that InfoNCE rejects, unless an explicit alignment price exceeds the gain of the nuisance channel. Finite-batch estimation adds a further difference. Our main contributions are as follows.
\begin{itemize}[leftmargin=*,itemsep=1pt,topsep=2pt]
\item \textbf{Population guarantees} (\cref{sec:theory}). SG vanishes exactly when the views agree almost surely and their common law is uniform. Under equal marginals it controls excess InfoNCE at a square-root rate whose dimension-free leading constant is attained by an explicit extremal family (\cref{thm:infonce}), and no mean-embedding (Hilbertian) pair discrepancy with this target can do better (\cref{thm:barrier}). With an explicit alignment term the certificate becomes linear, and a spectral criterion, condition (U), characterizes which uniformity tests allow this (\cref{thm:spectral}).
\item \textbf{Finite-batch estimation} (\cref{sec:finite}). The plug-in estimator inherits SIGReg's finite-sample bias, and the centroid turns it into a preference for imperfect alignment. An off-diagonal U-statistic removes the bias at $O(BdM)$ cost without a pairwise similarity matrix (\cref{prop:bias,prop:unbiased}).
\item \textbf{Nuisance preferences and the alignment price} (\cref{sec:nuisance}). Population agreement does not fix preferences away from the optimum. Along an isotropic channel, per-view nuisance lowers SG at first order whenever the shared code is non-uniform, while population InfoNCE rejects it near uniformity (\cref{thm:slaving}). An alignment price (SG$_\alpha$), $\J_\alpha=\J+\alpha\Delta$ with $\Delta$ the expected misalignment, makes the nuisance-free solution a strict local minimizer when $\alpha$ exceeds the channel's gain and not when it falls below; the same rule applies to two-term objectives such as LeJEPA (\cref{prop:condA}).
\item \textbf{Experiments} (\cref{sec:experiments}). In controlled latent models, SG$_\alpha$ comes close to InfoNCE without per-view nuisance, and Gaussian uniformity tests that satisfy (U) reach InfoNCE's uniformity in training while moment matching does not. With nuisance, pure SG retains or increases it, and an alignment price above the gain measured on the trained encoders removes it; for LeJEPA at three batch sizes the measured gain separates the encoders that keep style from those that do not. InfoNCE training also reaches lower values of the SG$_{0.2}$ objective than SG$_{0.2}$ trained from scratch, which points to an \emph{optimization gap}.
\end{itemize}

\section{Related Work}

\paragraph{Alignment and uniformity.} Our comparison builds on the decomposition of InfoNCE, in the limit of infinitely many negatives, into alignment and uniformity \citep{wang2020understanding}; see \citet{wang2021behaviour} for the role of temperature. Downstream guarantees \citep{arora2019theoretical,haochen2021spectral} and identifiability results for contrastive \citep{zimmermann2021cl,vonkugelgen2021content,rusak2025infonce} and LeJEPA-type objectives \citep{klindt2026lejepa,nicollier2026beyond} characterize what the optimum represents. \citet{betser2026infonce} show that, under alignment and concentration assumptions, low-dimensional projections of InfoNCE representations become Gaussian as the dimension grows. We instead compare a distribution test with InfoNCE away from the optimum, through a sharp bound on InfoNCE's excess that holds for every law with equal marginals and through local preferences.

\paragraph{Gaussian and kernel regularizers.} SIGReg \citep{balestriero2025lejepa} matches sliced characteristic functions of each view to $\mathcal N(0,I)$ \citep{epps1983test,baringhaus1988consistent} and adds a predictive alignment term; KerJEPA \citep{zimmermann2025kerjepa} recasts it as a sliced MMD; we use its kernel form and SIGReg's bias identity directly. SFRIK \citep{zheng2023sfrik} and SPHERE-JEPA \citep{nicollier2026sphere,nicollier2026expanding} use zonal or sliced uniformity tests on the sphere, which are Sobolev tests of uniformity \citep{gine1975invariant,garciaportugues2018overview}. Reported limitations of SIGReg-type regularizers include vanishing gradients \citep{yu2026qqworld,wu2026visreg} and failure to prevent spurious attributes from structuring the latent space \citep{nicollier2026uowreg}, and one-view marginals do not identify the coupling of the views \citep{jenkinsonalvarez2026beyondisotropy}; our gain--price rule (\cref{sec:alignment-price}) gives the invariance weight at which a two-term objective nevertheless rejects a given nuisance channel. A sphere--interval decomposition of the Gaussian also appears in \citet{parakhin2026wristband}. These regularizers test marginal embedding distributions, whereas SG tests a statistic of the pair. Testing the pair gives InfoNCE's exact target and a sharp comparison with it, and it turns SIGReg's finite-sample bias into a bias toward misalignment.

\paragraph{Contrastive and non-contrastive.} \citet{garrido2023duality} relate covariance criteria to sample-Gram criteria, both second-moment objects, as is the whitening of W-MSE \citep{ermolov2021wmse}; SSL-HSIC maximizes a kernel dependence \citep{li2021sslhsic}, and \citet{fang2024rethinking} propose a Wasserstein uniformity metric; MMCR \citep{yerxa2023mmcr,schaeffer2024mmcr} also uses the norm of the view centroid, with a nuclear-norm objective. \citet{mikulasch2026ldm} unify contrastive, non-contrastive and predictive objectives as latent distribution matching, \citet{chen2024secretly} recast InfoNCE as optimal-transport alignment, and \citet{dufumier2023kernel} replace its uniformity term with a kernel-based decoupled uniformity. We compare targets, certificates and local preferences quantitatively, and show that second-moment matching cannot control InfoNCE's uniformity term (\cref{cor:examples}).

\section{Shared Gaussianization}\label{sec:setup}

Let $d\ge2$, $\Sd\subset\R^d$ be the unit sphere, and $\sigma$ be its uniform probability measure. A \emph{paired law} $P$ is the joint law of $(U,V)\in\Sd\times\Sd$, with marginals $q_U,q_V$ and mixture marginal $\bar q=\frac12(q_U+q_V)$; $P$ has \emph{equal marginals} if $q_U=q_V=:q$. We measure misalignment by
\begin{equation}\label{eq:delta}
\Delta(P)=\tfrac14\,\E\|U-V\|^2=\tfrac12\big(1-\E\,U^\top V\big)\in[0,1].
\end{equation}

\paragraph{The shared loss.} Given $(U,V)\sim P$, let
\begin{equation}\label{eq:construction}
C=\tfrac12(U+V),\qquad Z=RC,\qquad R\sim\chi_d,
\end{equation}
with $R$ independent of $(U,V)$, and write $\varphi_Z(\omega)=\E\,e^{i\omega^\top Z}$. SG matches the law of $Z$ to $\mathcal N(0,I_d)$ through
\begin{equation}\label{eq:J}
\J(P)=\E_{a\sim\sigma,\,t\sim\mathcal N(0,1)}\big|\varphi_Z(ta)-e^{-t^2/2}\big|^2 ,
\end{equation}
where $e^{-t^2/2}$ is the standard Gaussian characteristic function at $ta$ ($\|a\|=1$). This is the sliced Epps--Pulley statistic of SIGReg \citep{epps1983test,balestriero2025lejepa}, applied to a statistic of the pair instead of to each view; equivalently, $\J$ is the squared MMD \citep{gretton2012kernel} between the law of $Z$ and $\mathcal N(0,I_d)$ for the kernel ${}_1F_1(\frac12;\frac d2;-\|x-y\|^2/2)$ \citep[Thm.~7]{zimmermann2025kerjepa}. If $U=V\sim\sigma$, then $Z=RU\sim\mathcal N(0,I_d)$; otherwise $\|C\|<1$ on disagreeing pairs. The centroid is not renormalized since its norm carries the alignment signal. SG itself is defined for arbitrary paired laws; equal marginals are needed only for the comparison with InfoNCE.

\paragraph{Population InfoNCE.} Let $P$ have equal marginals $q$ and let $\beta=1/\tau>0$. InfoNCE with $N$ negatives, minus $\log N$, converges as $N\to\infty$ to \citep{wang2020understanding}
\begin{equation}\label{eq:infonce}
\Lb(P)=-\beta\,\E\,U^\top V+\E_{U\sim q}\log\E_{W\sim q}e^{\beta U^\top W}.
\end{equation}
With $\kappa_d(\beta)=\E_{W\sim\sigma}e^{\beta u^\top W}$, which does not depend on $u\in\Sd$, and $\mathcal L^\star_\beta=-\beta+\log\kappa_d(\beta)$, the excess separates into a misalignment cost and a uniformity excess (\cref{lem:decomp}):
\begin{equation}\label{eq:infonce-decomp}
\Lb(P)-\mathcal L^\star_\beta=2\beta\Delta(P)+\Ub(q),
\end{equation}
where $\Ub(q)=\E_{U\sim q}\log\E_{W\sim q}e^{\beta U^\top W}-\log\kappa_d(\beta)\ge0$, with equality only at $q=\sigma$ (\cref{lem:unif}).
Our question is whether the single discrepancy $\J(P)$ controls both terms. For NT-Xent, whose negatives are pooled from both views, the population limit is $\Lb(\tilde P)$, where $\tilde P$ swaps the coordinates with probability $\frac12$; $\tilde P$ has equal marginals $\bar q$ and the same $\Delta$ and $\J$, so equal marginals suffice for this symmetrized comparison.

\begin{figure*}[!t]
\centering
\includegraphics[width=\textwidth]{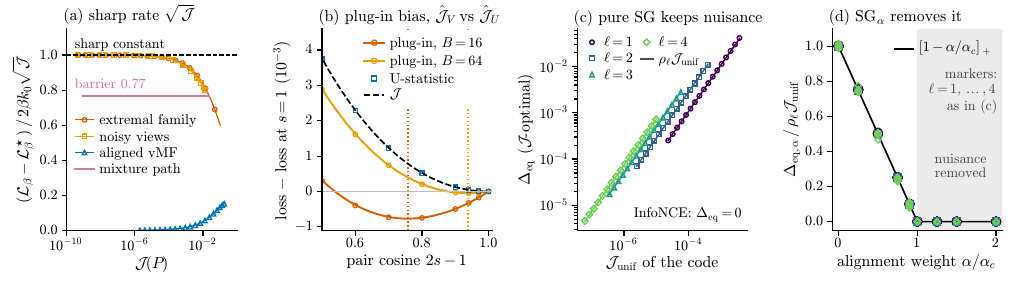}
\caption{\textbf{Theory on controlled families} (\cref{app:exp-details}). (a) InfoNCE excess relative to the leading term $2\beta k_0\sqrt\J$ of \cref{thm:infonce} ($d=16$, $\beta=2$): misaligned families attain the sharp constant, the mixture path of \cref{thm:barrier} stays at $0.77$, and the aligned von Mises--Fisher (vMF) path has excess $\Theta(\J)$. (b) The plug-in estimator is minimized away from alignment (dotted), whereas the U-statistic follows $\J$ (extremal family, $d=16$; markers: Monte Carlo). (c) Along the isotropic channel, pure SG's optimal misalignment for residual non-uniformity of degree $\ell$ follows $\rho_\ell\J_{\rm unif}$ (\cref{thm:slaving}, $d=8$; $\rho_\ell$ as after \eqref{eq:balance}); InfoNCE's is zero. (d) With $\J_\alpha$, the rescaled misalignment collapses onto $[1-\alpha/\alpha_c]_+$, where $\alpha_c=g$ is the channel's gain.}
\label{fig:theory}
\end{figure*}

\section{Population Guarantees for SG}\label{sec:theory}

Because view disagreement shortens the centroid (\cref{sec:setup}), we can identify the exact population minimizers of SG and obtain a sharp comparison with population InfoNCE, at a rate that no Hilbertian pair discrepancy with InfoNCE's target can improve. Once misalignment is charged linearly, a spectral condition decides which uniformity tests give a linear certificate. Proofs are in \cref{app:proofs}.

\subsection{Alignment from the Centroid Radius}\label{sec:alignment}\label{sec:sg}

Let $H=\frac12(U-V)$ and $Q=\|H\|^2$. Then $\|C\|^2+Q=1$ and $\E Q=\Delta(P)$, so disagreement appears as a deficit in the centroid radius. For a fixed centroid, averaging over the uniform direction $a$ and the radius turns $Ra$ into a standard Gaussian vector, hence
\begin{equation}\label{eq:radial-cf}
\E_{a\sim\sigma}\varphi_Z(ta)=\E\,e^{-t^2\|C\|^2/2}=e^{-t^2/2}\,\E\,e^{t^2Q/2}.
\end{equation}

\begin{lemma}[Alignment control]\label{lem:align}
Let $j(\Delta)=(3-2\Delta)^{-1/2}-2(3-\Delta)^{-1/2}+3^{-1/2}$, which is increasing on $[0,1]$. For every paired law $P$,
\[
\J(P)\ \ge\ j\big(\Delta(P)\big)\ \ge\ \frac{\Delta(P)^2}{12\sqrt3},
\]
hence $\Delta(P)\le k_0\sqrt{\J(P)}$ with $k_0=2\cdot3^{3/4}\approx4.56$.
The first inequality is an equality on an extremal family of pairs at a fixed angle about a uniform centroid direction (\cref{sec:infonce-comparison}).
\end{lemma}

\emph{Proof sketch.} Jensen's inequality over $a$ inside the modulus, \eqref{eq:radial-cf}, and convexity of $x\mapsto e^{t^2x/2}$ with $\E Q=\Delta$ (both sides are $\ge1$, so squaring preserves the order) give
\begin{align*}
\J(P)&\ge\E_t\big[e^{-t^2}\big(\E e^{t^2Q/2}-1\big)^2\big]\\
&\ge\E_t\big[e^{-t^2}\big(e^{t^2\Delta/2}-1\big)^2\big]=j(\Delta)\\
&\ge\tfrac{\Delta^2}{4}\,\E_t\big[t^4e^{-t^2}\big]=\tfrac{\Delta^2}{12\sqrt3}.
\end{align*}
Since $j(\Delta)=\Delta^2/(12\sqrt3)+O(\Delta^3)$, the alignment floor rises only quadratically, and on the extremal family, which attains it, so does $\J$; this is the source of the square root below.

\subsection{Exact Population Minimizers}\label{sec:minimizers}\label{sec:ideal}

\begin{proposition}[Exact population minimizers]\label{prop:sg-minimizers}
For every paired law $P$, $\J(P)=0$ if and only if $U=V$ almost surely and $U\sim\sigma$.
\end{proposition}

The alignment bound already forces agreement; the Gaussian target then determines the distribution of their common direction.

\emph{Proof.} If $\J=0$, \cref{lem:align} gives $U=V$ almost surely. Since $\omega=ta$ covers $\R^d$ and characteristic functions are continuous, $\varphi_Z(\omega)=e^{-\|\omega\|^2/2}$ for all $\omega$, so $Z=RU\sim\mathcal N(0,I_d)$; as $R>0$, $U=Z/\|Z\|\sim\sigma$. The converse is the polar decomposition. $\square$

Thus SG and population InfoNCE have the same minimizers, the aligned, uniform laws. For two-view objectives that test each view separately, the target depends on the test. Let $X_1,X_2\in\R^d$ be unnormalized view embeddings. Write $I(X_1,X_2)$ for a nonnegative population invariance loss that vanishes exactly when $X_1=X_2$ almost surely, and $S(X)$ for a nonnegative characteristic population Gaussian discrepancy, integrated over all directions and a full-support frequency law, that vanishes exactly when $X\sim\mathcal N(0,I_d)$. By population VICReg we mean a positive weighted sum of the two-view invariance loss, coordinatewise variance hinges with target standard deviation one, and squared off-diagonal covariance penalties, all evaluated at population moments.

\begin{proposition}[Directional targets]\label{prop:ideal}\label{thm:ident}
\leavevmode
\begin{enumerate}[label=(\roman*),leftmargin=*,itemsep=1pt,topsep=2pt]
\item For $\lambda\in(0,1)$, the idealized population objective $(1{-}\lambda)\,I(X_1,X_2)+\frac\lambda2\{S(X_1)+S(X_2)\}$ is zero if and only if $X_1=X_2$ almost surely and the common embedding is $\mathcal N(0,I_d)$. Its normalized directions therefore have the aligned, uniform target of \cref{prop:sg-minimizers}.
\item Population VICReg's zero set also contains aligned embeddings uniform on the cross-polytope $\{\pm\sqrt d\,e_i:i\le d\}$. These embeddings have mean zero and covariance $I_d$, while their directions are non-uniform.
\end{enumerate}
\end{proposition}

Thus an idealized population LeJEPA objective (exact invariance, no predictor) has the same normalized target as SG and InfoNCE, whereas population VICReg, which constrains only the first two moments, also admits aligned embeddings with non-uniform directions. \Cref{sec:experiments} finds the same contrast in trained encoders.

\subsection{A Sharp Comparison with Population InfoNCE}\label{sec:infonce-comparison}\label{sec:bridge}

Sharing minimizers gives only qualitative agreement. Under equal marginals, SG also controls the excess population InfoNCE loss.

\begin{theorem}[Sharp comparison with population InfoNCE]\label{thm:infonce}\label{thm:infonce-full}
Let $P$ have equal marginals and $\beta>0$. Then
\begin{equation}\label{eq:main}
0\le\Lb(P)-\mathcal L^\star_\beta\le2\beta k_0\sqrt{\J(P)}+\log\big(1+\bar c_{d,\beta}\J(P)\big),
\end{equation}
where $\bar c_{d,\beta}<\infty$ depends only on $d$ and $\beta$. Moreover, there is a family of paired laws with uniform marginals such that $\J(P)\to0$ and
\begin{equation}\label{eq:sharp-leading}
\Lb(P)-\mathcal L^\star_\beta=2\beta k_0\sqrt{\J(P)}+o\big(\sqrt{\J(P)}\big).
\end{equation}
Hence both the square-root order and the dimension-free leading constant $2\beta k_0$ are sharp.
\end{theorem}

Moreover, for laws with equal marginals, $\J(P_n)\to0$ if and only if $\Lb(P_n)-\mathcal L^\star_\beta\to0$ (\cref{app:infonce-proof}). Explicitly, $\bar c_{d,\beta}=(1+\frac32k_0)^2C^\star_{d,\beta}/\kappa_d(\beta)$ with the kernel-comparison constant $C^\star_{d,\beta}\le\sqrt6\,e^{4\beta}$ of \cref{prop:kernel}. Unlike $2\beta k_0$, $\bar c_{d,\beta}$ depends on $d$ and grows at least like $e^{\beta/4+o(\beta)}$ as $\beta\to\infty$ at fixed $d$; at $d=8$, $\beta=5$ it is about $4.9\times10^4$. Replacing $2\beta k_0\sqrt\J$ by $2\beta\Delta^\star(\J)$, with $\Delta^\star(\J)=j^{-1}(\J)$ for $\J\le j(1)$ and $1$ otherwise, gives a bound whose first term is exact on the extremal family; aligned laws have $\Delta=0$ and $\J=\Dlift^2(q,\sigma)$, so their excess is $\Ub(q)\le(C^\star_{d,\beta}/\kappa_d(\beta))\,\J$ (\cref{lem:decomp,prop:kernel}; \cref{fig:theory}a).

\emph{Proof sketch.} By \eqref{eq:infonce-decomp}, it suffices to bound $2\beta\Delta$ and $\Ub(q)$. \emph{Alignment:} \cref{lem:align} gives $\Delta\le k_0\sqrt\J$. \emph{Transfer:} let $\Dlift^2(q,\sigma)$, the \emph{lifted} discrepancy, be the SG loss of the aligned law with marginal $q$. Since $\frac12(e^{ix^\top U}+e^{ix^\top V})=e^{ix^\top C}\cos(x^\top H)$, the pair test controls the marginal (\cref{thm:structure}):
\begin{equation}\label{eq:marginal-transfer}
\big|\sqrt{\J(P)}-\Dlift(\bar q,\sigma)\big|\le\tfrac32\,\Delta(P).
\end{equation}
Under equal marginals, $\bar q=q$. \emph{Kernel comparison:} $\Dlift$ is an MMD on the sphere whose zonal kernel, like $e^{\beta u^\top v}$, diagonalizes in spherical harmonics; comparing eigenvalues degree by degree gives $\Db^2\le C^\star_{d,\beta}\Dlift^2$ for the MMD $\Db$ of $e^{\beta u^\top v}$ (\cref{prop:kernel}). \emph{Uniformity:} by Jensen's inequality and $\E_{W\sim\sigma}e^{\beta u^\top W}=\kappa_d(\beta)$, $\E_q\log\E_qe^{\beta U^\top W}\le\log\E_{q\otimes q}e^{\beta U^\top W}=\log(\kappa_d(\beta)+\Db^2(q,\sigma))$, so $\Ub(q)\le\log(1+\Db^2/\kappa_d(\beta))$ (\cref{lem:decomp}).

\paragraph{Sharpness.} Let $(n,h)$ be a Haar-distributed orthonormal pair in $\R^d$ and, for $s\in[0,1]$,
\[
U_s=\sqrt s\,n+\sqrt{1-s}\,h,\qquad V_s=\sqrt s\,n-\sqrt{1-s}\,h .
\]
Both marginals are uniform, $C_s=\sqrt s\,n$ and $\Delta=1-s$, so $Z_s\sim\mathcal N(0,sI_d)$ and $\J=j(\Delta)$, so this \emph{extremal family} attains \cref{lem:align}. Its uniformity excess vanishes, so $\Lb-\mathcal L^\star_\beta=2\beta\Delta$, and $j(\Delta)=\Delta^2/(12\sqrt3)+O(\Delta^3)$ gives \eqref{eq:sharp-leading}. The same family is used in \cref{prop:bias} and \cref{fig:theory}a,b.

\subsection{Optimality of the Square-Root Rate}\label{sec:sqrt-barrier}

The square-root rate is optimal for a broad class of pair discrepancies. Call a pair discrepancy $F$ \emph{Hilbertian} (a squared mean-embedding distance) if there exist a bounded feature map $\Psi$ into a Hilbert space $\mathcal H$, a target $m^\star\in\mathcal H$ and independent auxiliary randomness $\xi$ such that $F(P)=\|\E_P\Psi(U,V,\xi)-m^\star\|_{\mathcal H}^2$. This class contains $\J$, for any radius law and frequency weight, as well as every bounded-kernel squared MMD of a statistic of the pair.

\begin{proposition}[Square-root barrier]\label{thm:barrier}\label{prop:sqrt-barrier}
Let $F$ be Hilbertian with $F(P_0)=0$ for an aligned law $P_0$ with uniform marginals. Let $P_1$ have uniform marginals and $\Delta(P_1)>0$, and let $P_p=(1-p)P_0+pP_1$. Then $F(P_p)=p^2F(P_1)$, while $\Lb(P_p)-\mathcal L^\star_\beta=2\beta p\,\Delta(P_1)$. Hence either $F$ vanishes on a misaligned law, or no bound $\Lb-\mathcal L^\star_\beta\le\psi(F)$ with $\psi(x)=o(\sqrt x)$ holds.
\end{proposition}

\emph{Proof.} $F(P_0)=0$ gives $\E_{P_0}\Psi=m^\star$, so $\E_{P_p}\Psi-m^\star=p(\E_{P_1}\Psi-m^\star)$ and $F(P_p)=p^2F(P_1)$. The marginals of $P_p$ stay uniform, so $\Ub=0$, and since $\Delta$ is linear in $P$ with $\Delta(P_0)=0$, the excess is $2\beta p\Delta(P_1)$ by \eqref{eq:infonce-decomp}. If $F(P_1)>0$, the excess equals $\big(2\beta\Delta(P_1)/\sqrt{F(P_1)}\big)\sqrt{F(P_p)}$ for all $p$, which no $\psi(x)=o(\sqrt x)$ dominates. $\square$

A mean embedding sees a fraction $p$ of misaligned pairs at order $p^2$, whereas InfoNCE sees it at order $p$; along the mixture path only the coupling changes. Thus no Hilbertian pair discrepancy controls the InfoNCE excess faster than $\sqrt F$, and SG attains this rate with the sharp constant of \cref{thm:infonce}; for SG the order comes from \cref{lem:align}, since $\J$ responds to misalignment only at second order.

\subsection{Which Uniformity Tests Control InfoNCE Linearly}\label{sec:linear}

Charging misalignment directly removes the barrier. With the alignment price $\J_\alpha=\J+\alpha\Delta$ of \cref{sec:alignment-price}, under equal marginals, $\Lb-\mathcal L^\star_\beta\le c_{\beta,\alpha}\J_\alpha$ for every $\alpha>0$, with explicit $c_{\beta,\alpha}$ (\cref{prop:calibrated}). More generally, once misalignment is paid for linearly, a linear certificate needs a test of the marginal that controls InfoNCE's uniformity term $\Ub$ linearly. Which uniformity tests do?

\begin{figure}[!t]
\centering
\includegraphics[width=\columnwidth]{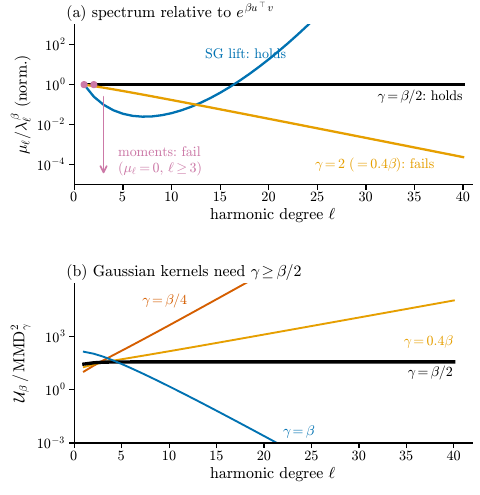}
\caption{\textbf{Condition (U)} ($d=8$, $\beta=5$). (a) Eigenvalues $\mu_\ell$ of zonal squared MMDs relative to those of $e^{\beta u^\top v}$, normalized at $\ell=1$: the ratio is bounded below for SG's lift and Gaussian kernels with $\gamma\ge\beta/2$, and not for $\gamma=2$ or moment matching ($\mu_\ell=0$ for $\ell\ge3$). (b) Exact $\Ub$ over the squared MMD of $e^{-\gamma\|u-v\|^2}$ for a law perturbed in degree $\ell$; bounded in $\ell$ if and only if $\gamma\ge\beta/2$ (\cref{cor:examples}).}
\label{fig:spectral}
\end{figure}

Squared MMDs to $\sigma$ with zonal (rotation-invariant) kernels are classical Sobolev statistics of uniformity; they diagonalize in spherical harmonics, and their weights decide which departures from uniformity they detect \citep{gine1975invariant,garciaportugues2018overview,garciaportugues2023projection}. $\Ub(q)=\int\log\big(\int k_\beta(\cdot,w)\,\dd q(w)\big)\,\dd q-\log\kappa_d(\beta)$, with $k_\beta(u,v)=e^{\beta u^\top v}$, is not of this form, because the logarithm sits inside the outer expectation. Let $\lambda^\beta_\ell$ be the degree-$\ell$ eigenvalues of $k_\beta$, so $\lambda^\beta_0=\kappa_d(\beta)$, and let $\Db$ be the MMD of $k_\beta$. For a zonal kernel $K$ with eigenvalues $(\mu_\ell)$, $\mathrm{MMD}_K^2(q,\sigma)=\sum_{\ell\ge1}\mu_\ell\|\hat q_\ell\|^2$, where $\hat q_\ell$ are the harmonic coefficients of $q$ (\cref{app:spectral}).

\begin{theorem}[InfoNCE uniformity is spectral]\label{thm:spectral}
Fix $d\ge2$, $\beta>0$, and write $\kappa=\kappa_d(\beta)$.
\begin{enumerate}[label=(\roman*),leftmargin=*,itemsep=1pt,topsep=2pt]
\item Let $\delta:=\|\dd q/\dd\sigma-1\|_\infty\le\frac12$ and
\[
\mathcal Q_\beta(q)=\sum_{\ell\ge1}\frac{\lambda^\beta_\ell}{\kappa}\Big(1-\frac{\lambda^\beta_\ell}{2\kappa}\Big)\|\hat q_\ell\|^2 .
\]
Then $\frac12\Db^2/\kappa\le\mathcal Q_\beta\le\Db^2/\kappa$ and $|\Ub(q)-\mathcal Q_\beta(q)|\le\frac32\,\delta\,\Db^2(q,\sigma)/\kappa$.
\item For a continuous positive-definite zonal kernel $K$ with eigenvalues $(\mu_\ell)$, the best constant $c_K$ in $\Ub\le c_K\,\mathrm{MMD}_K^2$ satisfies
\[
\sup_{\ell\ge1}\frac{\lambda^\beta_\ell(1-\lambda^\beta_\ell/2\kappa)}{\kappa\,\mu_\ell}\ \le\ c_K\ \le\ \sup_{\ell\ge1}\frac{\lambda^\beta_\ell}{\kappa\,\mu_\ell}=:\frac{S_K}{\kappa}.
\]
with the convention $x/0=\infty$ for $x>0$.
\end{enumerate}
\end{theorem}

Part (i) expands the log-partition functional to second order, with a remainder uniform in $d$ and $\beta$; near $\sigma$, $\Ub$ lies within a factor $2$ of the squared MMD with kernel $e^{\beta u^\top v}/\kappa$, up to an $O(\delta)$ error. Part (ii) combines this local necessity with the global bound $\Ub\le\log(1+\Db^2/\kappa)$ (\cref{lem:decomp}). A zonal test controls InfoNCE's uniformity term at temperature $1/\beta$ linearly if and only if its spectrum dominates that of $e^{\beta u^\top v}$, that is, $S_K<\infty$. We call this condition (U).

\begin{corollary}[Which tests satisfy (U)]\label{cor:examples}
For every $d\ge2$ and $\beta>0$:
\begin{enumerate}[label=(\alph*),leftmargin=*,itemsep=1pt,topsep=2pt]
\item The Gaussian lift of SG: $S_K=C^\star_{d,\beta}<\infty$, with $C^\star_{d,\beta}/\kappa_d(\beta)\ge e^{\beta/4+o(\beta)}$ as $\beta\to\infty$ with $d$ fixed.
\item Gaussian kernels $e^{-\gamma\|u-v\|^2}$: $S_K<\infty$ if and only if $\gamma\ge\beta/2$. The same threshold decides linear control of $\Ub$ by the excess of the uniformity loss $\log\E_{q\otimes q}e^{-\gamma\|U-U'\|^2}$ ($U'$ an independent copy) of \citet{wang2020understanding}.
\item Heat kernels, band-limited kernels, and first- and second-moment matching: $S_K=\infty$.
\end{enumerate}
\end{corollary}

Characteristic Gaussian and heat kernels still identify $\sigma$ when (U) fails, unlike finite-moment tests, but they control $\Ub$ only sublinearly, with a defect that grows with the harmonic degree (\cref{fig:spectral}). The squared MMD with the default $\gamma=2$ kernel of \citet{wang2020understanding}, whose uniformity loss it matches to first order near $\sigma$ up to a constant factor, therefore satisfies (U) exactly for $\tau\ge1/4$. Since $e^{-\gamma\|u-v\|^2}=e^{-2\gamma}e^{2\gamma u^\top v}$, the threshold says that the kernel must be at least as concentrated as InfoNCE's own kernel $e^{\beta u^\top v}$ (inverse temperature $2\gamma\ge\beta$). SG's lift passes at every temperature, with a constant that grows exponentially in $\beta$ (at $d=8$, $\beta=5$ the supremum in $S_K$ is attained at $\ell=7$, with $S_K/\kappa\approx790$), whereas the matched Gaussian kernel $\gamma=\beta/2$ has $S_K/\kappa=e^\beta/\kappa_d(\beta)=O(\beta^{(d-1)/2})$; a Gaussian kernel certifies only $\beta\le2\gamma$, so the lift trades a larger constant for validity at every temperature. With an alignment term, condition (U) is necessary and sufficient for a linear certificate.

\begin{corollary}[Linear certificates]\label{cor:linear}
Let $\alpha>0$ and let $K$ be a continuous positive-definite zonal kernel. Under equal marginals,
\[
\Lb-\mathcal L^\star_\beta\le\max\Big\{\frac{2\beta}{\alpha},\frac{S_K}{\kappa_d(\beta)}\Big\}\big(\alpha\Delta+\mathrm{MMD}^2_K(q,\sigma)\big),
\]
and no bound of this form holds if $S_K=\infty$.
\end{corollary}

\Cref{sec:experiments} tests this distinction in trained encoders.

\section{Finite-Batch Estimation}\label{sec:finite}

In training, we estimate $\J$ from a batch of $B$ i.i.d.\ pairs with independent radii, $Z_j=R_jC_j$, and $M$ fresh frequencies $\omega_m=t_ma_m$ with $a_m\sim\sigma$, $t_m\sim\mathcal N(0,1)$. Let $\hat\varphi_m=\frac1B\sum_je^{i\omega_m^\top Z_j}$ and $g_m=e^{-t_m^2/2}$. The plug-in estimator inherits SIGReg's finite-sample bias \citep[Thm.~6]{balestriero2025lejepa}; through the centroid, this bias changes the alignment incentive. Proofs are in \cref{app:finite-proofs}.

\subsection{Plug-in Bias Toward Misalignment}\label{sec:finite-bias}

\begin{proposition}[Plug-in bias favors misalignment]\label{prop:bias}
The plug-in estimator $\hat\J_V=\frac1M\sum_{m}|\hat\varphi_m-g_m|^2$ satisfies
\[
\E\hat\J_V=\J(P)+\frac1B\,\E_{a,t}\big[1-|\varphi_Z(ta)|^2\big].
\]
For every $B\ge2$, no aligned uniform law minimizes $\E\hat\J_V$: along the extremal family, $\E\hat\J_V$ has a unique minimizer, with pair cosine $\E\,U^\top V=1-4/B+O(B^{-2})$.
\end{proposition}

The bias comes from the diagonal terms: $\E\hat\varphi_m=\varphi_Z(\omega_m)$ and $\E|\hat\varphi_m|^2=(1-\frac1B)|\varphi_Z(\omega_m)|^2+\frac1B$, and expanding $|\hat\varphi_m-g_m|^2$ gives the identity. Shortening the centroid concentrates $Z$, which raises $|\varphi_Z|$ and lowers the bias term. On the extremal family, $|\varphi_Z(ta)|^2=e^{-t^2(1-\Delta)}$, so the bias is $\frac1B[1-(3-2\Delta)^{-1/2}]$, with slope $-3^{-3/2}/B$ at $\Delta=0$, against $j'(\Delta)\approx\Delta/(6\sqrt3)$; balancing the two gives $\Delta\approx2/B$, that is, pair cosine $1-4/B$ (\cref{fig:theory}b).

\subsection{An Unbiased U-Statistic Estimator}\label{sec:finite-unbiased}

Since $\E(1-|\hat\varphi_m|^2)=(1-\frac1B)(1-|\varphi_Z(\omega_m)|^2)$, removing the diagonal terms gives
\begin{equation}\label{eq:J-U}
\hat\J_U=\hat\J_V-\frac{1}{M(B-1)}\sum_{m=1}^M\big(1-|\hat\varphi_m|^2\big).
\end{equation}

\begin{proposition}[Unbiased SG estimator]\label{prop:unbiased}
For every paired law $P$, $B\ge2$ and $M\ge1$, $\E\hat\J_U=\J(P)$, where the expectation is over the batch, the radii and the frequencies. $\hat\J_U$ is a U-statistic over distinct pairs $j\ne k$ and is computed from the sums $\hat\varphi_m$ in $O(BdM)$ time and $O(BM+dM)$ memory, without a $B\times B$ similarity matrix.
\end{proposition}

The correction is the standard U-statistic; for SG, the bias it removes also shifts the alignment optimum. $\hat\J_U$ can be negative on a batch; clipping it would reintroduce bias, so we optimize it directly. In our experiments ($d=8$, $B=256$), SG$_\alpha$ (\cref{sec:alignment-price}, $\alpha=0.2$) with $M=B/2$ frequencies reaches the CKA of $M=256$ and $M=1024$ in all runs but one vMF seed, while $M=16$ is too few (\cref{app:exp-details}).

\section{Nuisance Preferences and the Alignment Price}\label{sec:nuisance}

SG and InfoNCE share their minimizers but can favor different representations away from the optimum. We exhibit a population failure mode that persists with exact expectations. When the shared code is non-uniform, pure SG can prefer independent per-view nuisance, because the nuisance improves the marginal distributional match faster than SG penalizes the resulting misalignment. Proofs are in \cref{app:slaving,app:gain}.

\subsection{An Isotropic Nuisance Channel}\label{sec:nuisance-channel}

Let $W\sim q$ be a shared code on $\Sd$ and, given $W$, draw $U_\vartheta,V_\vartheta$ independently from the heat kernel on $\Sd$ at time $\vartheta\ge0$. At $\vartheta=0$ the views coincide; for $\vartheta>0$ they carry independent perturbations, with $\Delta(\vartheta)=\frac12(1-e^{-2(d-1)\vartheta})$, and their marginal is smoothed toward $\sigma$. Let $\J_q(\vartheta)$ be the SG loss along this channel and $\J_{\rm unif}(q)=\J_q(0)=\Dlift^2(q,\sigma)$.

\subsection{Why Pure SG Can Prefer Nuisance}\label{sec:nuisance-preference}

The channel creates two competing effects. Independent perturbations shorten the centroid and hurt alignment, but they also smooth a non-uniform marginal toward the uniform distribution. For pure SG, the latter effect dominates to first order.

\begin{theorem}[Nuisance preference of pure SG]\label{thm:slaving}\label{thm:nuisance}
Along the isotropic nuisance channel:
\begin{enumerate}[label=(\roman*),leftmargin=*,itemsep=1pt,topsep=2pt]
\item For every $q\ne\sigma$, the right derivative satisfies $\partial_\vartheta\J_q|_{\vartheta=0}=-\sum_{\ell\ge1}G_\ell\|\hat q_\ell\|^2<0$, where $\hat q_\ell$ is the degree-$\ell$ harmonic component of $q-\sigma$ and $G_\ell>0$ is explicit. Thus a small amount of nuisance strictly decreases $\J$, and, for this fixed code $q$, the aligned law is not a local minimizer of $\J$ along the channel.
\item Let $q_\varepsilon=(1+\varepsilon f)\sigma$ with $f\ne0$ smooth and $\int f\,\dd\sigma=0$. There is $\vartheta_0(d)>0$ such that, for all small $\varepsilon$, $\J_{q_\varepsilon}$ has a unique minimizer on $[0,\vartheta_0]$, which is interior, and its misalignment $\Delta_{\rm eq}$ there satisfies $\Delta_{\rm eq}/\J_{\rm unif}(q_\varepsilon)\to\rho(f)$, an explicit constant with $\rho(f)>6\sqrt3$.
\item In the same setting, $\vartheta=0$ is the unique minimizer of $\Lb-\mathcal L^\star_\beta$ on $[0,\vartheta_0]$ for all small $\varepsilon$.
\end{enumerate}
\end{theorem}

Nuisance helps SG at first order because the centroid direction diffuses at half the speed of the views, which smooths the non-uniformity, and radial contraction reduces the remaining harmonic contributions. Misalignment is cheap, because along the channel the coefficient of $\Delta^2$ is $1/(12\sqrt3)$ to leading order, as in \cref{lem:align}. In the regime of (ii), therefore,
\begin{equation}\label{eq:balance}
\J\approx\J_{\rm unif}-g\,\Delta+\frac{\Delta^2}{12\sqrt3},\qquad g=\frac{\rho(f)\,\J_{\rm unif}}{6\sqrt3},
\end{equation}
where $g=-\dd\J/\dd\Delta|_{\Delta=0}$ is the channel's \emph{gain}. InfoNCE pays $2\beta\Delta$ at first order, while its uniformity improvement is of order $\varepsilon^2$; near uniformity InfoNCE therefore keeps no nuisance. For $f$ of pure degree $\ell$ write $\rho_\ell=\rho(f)$; at $d=8$, $\rho_1,\dots,\rho_4=10.8,28.5,49.5,73.8$; at equal $\J_{\rm unif}$, finer residuals induce more nuisance (\cref{fig:theory}c). The certificate still applies: $\Delta\le k_0\sqrt\J$ bounds the amount of nuisance that pure SG keeps.

\subsection{Adding an Alignment Price}\label{sec:alignment-price}

This mismatch suggests a minimal modification, SG$_\alpha$:
\begin{equation}\label{eq:J-alpha}
\J_\alpha(P)=\J(P)+\alpha\Delta(P),\qquad\alpha\ge0 .
\end{equation}
With \eqref{eq:balance}, $\J_\alpha\approx\J_{\rm unif}+(\alpha-g)\Delta+\Delta^2/(12\sqrt3)$. If $\alpha>g$, infinitesimal nuisance increases the objective and alignment is locally preferred; if $\alpha<g$, the incentive remains. In the regime of (ii), with $\alpha$ of order $\varepsilon^2$, the minimizer is $\Delta=[\rho(f)\J_{\rm unif}-6\sqrt3\,\alpha]_++o(\varepsilon^2)$, so the equilibrium nuisance vanishes for $\alpha\ge g$ (\cref{cor:threshold}, \cref{fig:theory}d). The price keeps SG's minimizers (\cref{prop:calibrated}).

The rule applies beyond the isotropic model. For a family $(P_\vartheta)_{\vartheta\in[0,\vartheta_0]}$ of paired laws with $\Delta(P_0)=0$, along which $\Delta$ and an objective $F$ are right-differentiable at $0$ with $\Delta'(0)>0$, call $g=-(\dd F/\dd\vartheta)/(\dd\Delta/\dd\vartheta)|_0$ the \emph{gain} of the family, and $\alpha>g$ \emph{condition (A)}. LeJEPA minimizes $\lambda S+(1-\lambda)I$ where $S$ is the average of the per-view SIGReg terms and $I$ is an invariance term; its reference implementation multiplies $S$ by the batch size $B$.

\begin{proposition}[Condition (A): gain against price]\label{prop:condA}
Let $g$ be the gain of a family as above.
\begin{enumerate}[label=(\roman*),leftmargin=*,itemsep=1pt,topsep=2pt]
\item If $\alpha>g$, $P_0$ is a strict local minimizer of $F+\alpha\Delta$ along the family; if $\alpha<g$, it is not.
\item For $\lambda S+(1-\lambda)I$ with $\lambda\in(0,1)$, along a family with $I(P_0)=0$, on which $S$ and $I$ are right-differentiable at $0$ with $\dd I/\dd\vartheta|_0>0$, (i) holds with $F$ replaced by $S$, $\Delta$ by $I$, $\alpha$ by $(1-\lambda)/\lambda$, and $g$ by $g_S=-\dd S/\dd I|_0$. Equivalently, for $g_S>0$, $P_0$ is a strict local minimizer if $\lambda<\lambda_c=1/(1+g_S)$ and is not one if $\lambda>\lambda_c$. Multiplying $S$ by $B$ replaces $g_S$ by $Bg_S$.
\end{enumerate}
\end{proposition}

At $\alpha=g$, higher-order terms decide. For LeJEPA, the price of misalignment is $(1-\lambda)/\lambda$, set against the gain $Bg_S$ of its batch-scaled regularizer; at fixed per-sample gain, the critical SIGReg weight $\lambda_c$ therefore decreases with $B$. The gain depends on the channel and can be measured on a trained encoder (\cref{sec:experiments}, \cref{fig:slaving}).

\paragraph{Beyond the isotropic channel.} If the two views are conditionally independent given a latent, with common conditional mean $c(\Lambda)W$, $c(\Lambda)\in[0,1]$, for a \emph{mean code} $W\sim q$ on $\Sd$, then $|\sqrt{\J(P)}-\sqrt{\J_{\rm unif}(q)}|\le6.5\,\Delta(P)$. Hence every such family with fixed $q$ that starts at the aligned law $U=V=W$ has gain at most $13\sqrt{\J_{\rm unif}(q)}$ (\cref{prop:centered,cor:centered-gain}), and heat channels whose diffusion time varies with the code point attain this order (\cref{prop:hetero}). Condition (A) thus becomes easier to meet as the code becomes uniform.

\begin{figure*}[!t]
\centering
\includegraphics[width=\textwidth]{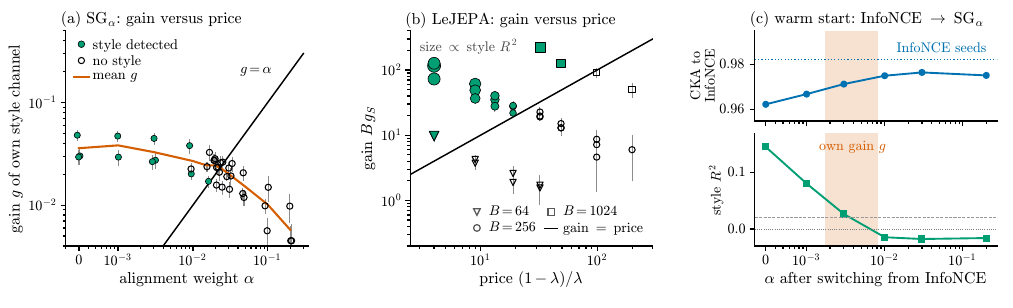}
\caption{\textbf{Gain against price in trained encoders} (content--style model, $d=8$; filled: style detected). (a) SG$_\alpha$: style is detected only when the own-channel gain exceeds $\alpha$. (b) LeJEPA, three batch sizes: style is detected exactly when $Bg_S$ exceeds the price $(1-\lambda)/\lambda$. (c) Warm starts stay close to InfoNCE; the style added by pure SG vanishes once $\alpha$ exceeds their gains (shaded; dashed: detection threshold).}
\label{fig:slaving}
\end{figure*}

\section{Experiments}\label{sec:experiments}\label{sec:latent}

We test \cref{sec:theory,sec:finite,sec:nuisance} in trained encoders. We ask whether SG reaches InfoNCE's representation when the target is the only difference, whether uniformity tests that satisfy (U) reach InfoNCE's uniformity, whether pure SG prefers per-view nuisance, and whether the measured gain decides which SG$_\alpha$ and LeJEPA encoders keep it.

\paragraph{Setup.} We use latent-variable models with known latent content \citep{zimmermann2021cl,vonkugelgen2021content}. In the \emph{content--style} model, a fixed random injective network maps a \emph{content} $c\sim\sigma$ on $\mathbb S^{7}$, shared by the two views, and an independent per-view \emph{style} $s\sim\mathcal N(0,I_4)$ to an observation $x=m(c,s)$. An encoder that recovers $c$ up to a measure-preserving map and ignores $s$ is aligned and uniform, so style is a per-view nuisance channel and the content plays the role of the shared code. In the von Mises--Fisher (\emph{vMF}) model of \citet{zimmermann2021cl} there is no style, and the second view's content is $\tilde c\sim\mathrm{vMF}(c,10)$. We train the same MLP encoder ($d=8$, $B=256$, 8k steps, 3 seeds; \cref{app:exp-details}) with InfoNCE ($\beta=5$), SG and SG$_\alpha$ (with $\hat\J_U$), LeJEPA (SIGReg weight $\lambda$) and VICReg. To separate the objective from the optimization path, we also switch InfoNCE encoders to SG$_\alpha$ halfway through training (warm start). The \emph{gain} of a trained encoder is measured on its own style channel. We rescale the style input by $\eta$ and fit $F-F(0)\approx-g\Delta+c_2\Delta^2$ for small $\eta$, with $F=\hat\J_U$ (for LeJEPA, the per-sample SIGReg term against $I$), averaging over independent repetitions. Style is \emph{detected} when a held-out probe from $(U,c)$ to the style reaches $R^2>0.02$; CKA is linear centered kernel alignment \citep{kornblith2019similarity} to InfoNCE encoders of other seeds.

\begin{table}[t]
\caption{Controlled study ($d=8$, $\beta=5$; mean over 3 seeds). CKA to InfoNCE encoders of other seeds; style: held-out probe $R^2$ (detected if $>0.02$); $\J$ and $\J_{0.2}=\J+0.2\Delta$, in units of $10^{-3}$, are evaluated on held-out pairs for every row, whatever the training loss. Warm starts switch InfoNCE to SG or SG$_{0.2}$ after 4k of 8k steps. Standard deviations: \cref{tab:latent-sd,tab:vmf}.}\label{tab:latent}
\centering\footnotesize
\setlength{\tabcolsep}{2.4pt}
\begin{tabular}{@{}lcccccc@{}}
\toprule
 & vMF & \multicolumn{5}{c}{content--style}\\
\cmidrule(lr){2-2}\cmidrule(l){3-7}
Method & CKA & CKA & style & $\Ub$ & $\J$ & $\J_{0.2}$\\
\midrule
\multicolumn{7}{@{}l}{\emph{from scratch}}\\
InfoNCE & $0.997$ & $0.982$ & $-0.02$ & $0.02$ & $0.17$ & $1.7$\\
SG & $0.938$ & $0.646$ & $0.11$ & $0.05$ & $0.85$ & $20.6$\\
SG$_{0.2}$ & $0.993$ & $0.828$ & $-0.02$ & $0.16$ & $0.98$ & $3.6$\\
LeJEPA, $\lambda=0.01$ & $0.981$ & $0.696$ & $-0.02$ & $0.26$ & $1.57$ & $6.0$\\
VICReg & $0.811$ & $0.597$ & $-0.02$ & $1.38$ & $95$ & $97$\\
\midrule
\multicolumn{7}{@{}l}{\emph{warm start from InfoNCE (4k steps)}}\\
switch point & -- & $0.969$ & $0.03$ & $0.05$ & $1.27$ & $3.9$\\
$\to$ SG & -- & $0.962$ & $0.14$ & $0.02$ & $0.12$ & $4.0$\\
$\to$ SG$_{0.2}$ & -- & $0.975$ & $-0.02$ & $0.04$ & $0.27$ & $2.2$\\
\bottomrule
\end{tabular}

\end{table}

\paragraph{Without per-view nuisance.} With noisy positives and no per-view style (vMF), SG$_{0.2}$ and LeJEPA learn representations close to InfoNCE's: CKA $0.993$ and $0.981$, against $0.997$ between InfoNCE seeds (\cref{tab:latent}), with content $R^2_c$ of $0.99$ and $0.97$ (\cref{tab:vmf}). Pure SG follows at $0.938$ ($0.935$ with $\hat\J_V$, as the bias of \cref{prop:bias} is $O(1/B)$), and VICReg at $0.811$. On content--style, VICReg aligns the views (pair cosine $0.979$, \cref{tab:latent-sd}) but remains far from uniform ($\Ub=1.38$ against $0.02$), in line with its larger target set (\cref{prop:ideal}). LeJEPA and VICReg were selected on small grids (\cref{app:exp-details}), so their rows are not fully tuned.

\begin{figure}[!t]
\centering
\includegraphics[width=\columnwidth]{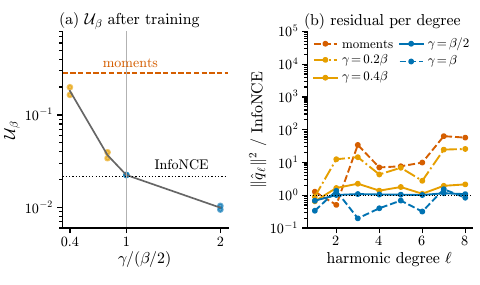}
\caption{\textbf{Condition (U) in trained encoders} (alignment term $+$ uniformity term; content--style model, $d=8$, $\beta=5$). (a) $\Ub$ of encoders trained with the Gaussian uniformity loss at scale $\gamma$ (one point per seed; blue: $\gamma\ge\beta/2$; amber: $\gamma<\beta/2$; gray line: mean over seeds) and mean over 3 seeds for moment matching (dashed) and InfoNCE (dotted). (b) Residual non-uniformity $\|\hat q_\ell\|^2$ per harmonic degree, relative to InfoNCE's (mean over seeds; dotted: equal to InfoNCE).}
\label{fig:utrain}
\end{figure}

\paragraph{Condition (U) in training.} \Cref{sec:linear} separates uniformity tests that permit a linear certificate from those that do not, and we check whether the distinction persists in training. We train the Gaussian uniformity loss of \citet{wang2020understanding}, with alignment term $\frac\beta2\E\|U-V\|^2=2\beta\Delta$, at $\gamma\in\{1,2,2.5,5\}$ (2 seeds), and first- and second-moment matching with alignment term $\Delta$ (3 seeds; \cref{app:exp-details}). All of them align the views (pair cosine $0.97$--$0.99$). With $\gamma\ge\beta/2=2.5$ the uniformity excess reaches InfoNCE's ($\Ub=0.022$ and $0.010$, against $0.021$ for InfoNCE in these runs), at $\gamma=2$ it is $0.037$, and at $\gamma=1$ it is $0.18$. Moment matching, which fails (U) at every temperature, keeps $\Ub=0.29$, and its residual non-uniformity is $7$--$63$ times InfoNCE's at every degree $\ell=3,\dots,8$ (\cref{fig:utrain}).

\paragraph{Pure SG prefers nuisance.} From scratch, pure SG keeps style ($R^2=0.11$; pair cosine $0.80$). The warm start isolates the preference. The half-trained InfoNCE encoders carry weak style ($R^2=0.03$; two of three seeds), which continuing with InfoNCE or with SG$_{0.2}$ removes ($-0.02$); continuing with SG raises it to $0.14$. With the same start and budget, the SG branch reaches the lowest $\J$ in \cref{tab:latent}, $0.12\times10^{-3}$, against $0.27\times10^{-3}$ for the style-free SG$_{0.2}$ branch and $0.17\times10^{-3}$ for fully trained InfoNCE (\cref{tab:latent}), in line with \cref{thm:slaving}(i).

\paragraph{The alignment price.} Across $38$ SG and SG$_\alpha$ encoders, every encoder with detected style has a gain above its price (\cref{fig:slaving}a). Style is detected in some seeds up to $\alpha=0.015$ and in none from $\alpha=0.02$, so the transition lies just below the crossing of the mean gain with $\alpha$ ($\alpha\approx0.023$). Warm starts compare prices from a common representation. Style is still detected at $\alpha=0.003$ (mean $R^2=0.027$) and is removed for $\alpha\ge0.01$, so the transition overlaps the range of the warm-started encoders' own gains, $[0.002,0.008]$ (\cref{fig:slaving}c).

\paragraph{Beyond SG: LeJEPA.} For LeJEPA, \cref{prop:condA}(ii) gives the critical weight $\lambda_c=1/(1+Bg_S)$. Across $34$ encoders and three batch sizes, style is retained by exactly those encoders whose measured gain $Bg_S$ exceeds the price $(1-\lambda)/\lambda$ (\cref{fig:slaving}b). The observed boundary moves from $\lambda\in(0.03,0.05)$ at $B=256$ to $(0.1,0.2)$ at $B=64$ (two seeds) and to $(0.01,0.02)$ at $B=1024$ (one seed). A single weight trades nuisance rejection against uniformity. At $\lambda=0.01$, LeJEPA discards style but has $\Ub=0.26$ (InfoNCE: $0.02$), and at $\lambda=0.1$ it keeps style ($R^2=0.52$; \cref{tab:latent-sd}). The finite-batch bias of the plug-in SIGReg used in training (\cref{prop:bias}) changes no classification (\cref{app:exp-details}).

\paragraph{An optimization gap.} With per-view nuisance, SG$_{0.2}$ from scratch stays farther from InfoNCE (CKA $0.83$) than without nuisance ($0.993$), although it retains no style, so this gap is separate from the preference of \cref{thm:slaving}; sharing InfoNCE's minimizers does not ensure that training reaches InfoNCE's representation. With a shared optimizer, schedule, budget and peak learning rate, InfoNCE training reaches a lower value of the SG$_{0.2}$ objective than SG$_{0.2}$ training itself. In units of $10^{-3}$, $\J_{0.2}=1.7$ for InfoNCE against $3.6$ for SG$_{0.2}$ from scratch; from the common switch point ($3.9$), continuing SG$_{0.2}$ reaches $2.2$ and keeps InfoNCE's representation (CKA $0.975$; InfoNCE seeds $0.982$). Pure SG likewise ends above InfoNCE on its own objective ($\J=0.85$ against $0.17$). More frequencies do not move SG$_{0.2}$ toward InfoNCE (CKA $0.825$ at $M=1024$; \cref{app:exp-details}), and without nuisance the gap nearly closes (vMF, SG$_{0.2}$ against InfoNCE: $\J=1.42$ and $1.37$, $\J_{0.2}=32.0$ and $31.6$; \cref{tab:vmf}). SG$_{0.2}$ training thus stops at representations that its own objective scores worse than InfoNCE's, which suggests that part of the difference lies in optimization.

\section{Discussion}

A single centroid test shares InfoNCE's minimizers and controls its excess at the optimal rate among Hilbertian pair discrepancies, yet away from the optimum it can favor per-view nuisance, as our controlled experiments illustrate. Three practical consequences follow. First, a linear guarantee needs two ingredients, an explicit alignment term and a uniformity test whose spectrum dominates that of $e^{\beta u^\top v}$; Gaussian kernels with $\gamma\ge\beta/2$ and SG's lift qualify and moment matching does not; within the Gaussian family, and against moment matching, the difference is visible in trained encoders. Second, the gain of a nuisance channel can be measured on a trained encoder, and comparing it with the alignment weight checks whether that weight rejects the channel. Third, because LeJEPA's reference implementation scales SIGReg by the batch size, its critical SIGReg weight falls with $B$, so a weight that rejects nuisance at one batch size can admit it at a larger one.

\paragraph{Limitations and next steps.} Our guarantees are population statements (equal marginals, infinitely many negatives), and the nuisance results are local to a channel. The certificate's second term has a constant that grows exponentially in $\beta$; at $d=8$, $\beta=5$ it is about $4.9\times10^4$, so the second term exceeds the leading one for $10^{-6}\lesssim\J\lesssim2\times10^{-2}$, a range that contains the $\J$ of trained encoders. The gains in \cref{sec:experiments} are measured on the trained encoders, so the comparisons with the price check condition (A) at the end of training; predicting retention from a gain measured before training, and testing the rule on natural images and at scale, are next steps, as is identifying the cause of the optimization gap.

{\small
\bibliography{refs}
}

\clearpage
\onecolumn
\appendix
\crefalias{section}{appendix}\crefalias{subsection}{subappendix}\crefalias{subsubsection}{subsubappendix}
\begin{center}{\titlefont Appendix}\end{center}
\vspace{0.5em}

\section{Proofs}\label{app:proofs}

\subsection{Auxiliary statements}\label{app:deferred}

The following statements are cited in the main text; their proofs are given below.

\begin{theorem}[Structure]\label{thm:structure}
For every paired law, $|\sqrt{\J(P)}-\Dlift(\bar q,\sigma)|\le\tfrac32\Delta(P)$.
\end{theorem}

\noindent Proof in \cref{app:centroid}.

\begin{proposition}[Lifted kernel]\label{prop:kernel}
Let $\lambda^{\rm lift}_\ell$ be the eigenvalues of $K_{\rm lift}$ and $C^\star_{d,\beta}=\sup_{\ell\ge1}\lambda^\beta_\ell/\lambda^{\rm lift}_\ell$. Then $\Db^2\le C^\star_{d,\beta}\Dlift^2$, no smaller constant works, and $C^\star_{d,\beta}\le\sqrt6\,e^{\beta}\max_{\ell\ge1}(3\beta)^\ell/\ell!\le\sqrt6\,e^{4\beta}$.
Moreover, for fixed $d$, $\liminf_{\beta\to\infty}\beta^{-1}\log(C^\star_{d,\beta}/\kappa_d(\beta))\ge\frac14$.
\end{proposition}

\noindent Proof in \cref{app:kernel-proof}. For $d=8$ and $\beta=5$, the setting of \cref{sec:experiments}, the supremum is attained at $\ell=7$ and $C^\star_{d,\beta}/\kappa_d(\beta)\approx790$, so $\bar c_{d,\beta}\approx4.9\times10^4$ in \eqref{eq:main}; the leading term then dominates for $\J\lesssim10^{-6}$.

\begin{proposition}[Calibrated shared loss]\label{prop:calibrated}
For $\alpha>0$, $\J_\alpha=0$ if and only if $\J=0$, and under equal marginals $\Lb(P)-\mathcal L^\star_\beta\le c_{\beta,\alpha}\,\J_\alpha(P)$ with $c_{\beta,\alpha}=\max\{(2\beta+\frac92c_\beta)/\alpha,\,2c_\beta\}$ and $c_\beta=C^\star_{d,\beta}/\kappa_d(\beta)$.
\end{proposition}

\begin{corollary}[Calibration threshold]\label{cor:threshold}
In the setting of \cref{thm:slaving}(ii), let $\J_\alpha=\J+\alpha\Delta$ with $\alpha=a\varepsilon^2$ for a fixed $a\ge0$. Then $\J_\alpha$ has a unique minimizer on $[0,\vartheta_0]$ for small $\varepsilon$, and
\[
\Delta(\vartheta^\star_{\alpha})=\big[\rho(f)\,\J_{\rm unif}(q_\varepsilon)-6\sqrt3\,\alpha\big]_++o(\varepsilon^2).
\]
In particular $\vartheta=0$ is the unique minimizer if and only if $\alpha\ge\alpha_c:=\rho(f)\J_{\rm unif}(q_\varepsilon)/(6\sqrt3)$; for any fixed $\alpha>0$ this holds for all small $\varepsilon$.
\end{corollary}

\noindent Proofs in \cref{app:slaving,app:calibrated-proof}.

\subsection{Preliminaries}\label{app:prelim}

Throughout, $D=d/2$, $a\sim\sigma$, $t\sim\mathcal N(0,1)$, $R\sim\chi_d$ are mutually independent and independent of the pair $(U,V)$, and $g(t)=e^{-t^2/2}$. We write $\|F\|$ for the norm of $L^2(\sigma\otimes\mathcal N(0,1))$ applied to a function $F(a,t)$, so that $\J=\|\varphi_Z(t\,\cdot)-g\|^2$. For a probability measure $\mu$ on $\Sd$, $T_\mu=RW$ with $W\sim\mu$ independent of $R$, so $T_\sigma\sim\mathcal N(0,I_d)$ and $Z=T_q$ when $U=V\sim q$. With the feature map $\psi_u(a,t)=\E_R\,e^{itRa^\top u}$ we have $\varphi_{T_\mu}(ta)=\int\psi_u(a,t)\,\dd\mu(u)$; the \emph{Gaussian lift} is the kernel $K_{\rm lift}(u,v)=\langle\psi_u,\psi_v\rangle$ (Hermitian inner product; $K_{\rm lift}$ is real since $\psi_u(-a,t)=\overline{\psi_u(a,t)}$), and $\Dlift(\mu,\sigma)=\|\varphi_{T_\mu}(t\,\cdot)-g\|$ is the corresponding MMD (\cref{lem:lift-eig}). Symbols are collected in \cref{app:notation}. We use the following standard facts.

\begin{enumerate}[label=\textbf{(F\arabic*)},leftmargin=*]
\item\label{F1} \emph{Gaussian polar decomposition.} $Ra\sim\mathcal N(0,I_d)$. Conversely, if $G\sim\mathcal N(0,I_d)$ then $\|G\|\sim\chi_d$ and $G/\|G\|\sim\sigma$ are independent.
\item\label{F2} For $c\in\R^d$ and $t\in\R$: $\E_{a,R}\,e^{itRa^\top c}=e^{-t^2\|c\|^2/2}$ (by \ref{F1}).
\item\label{F3} For $\theta\ge0$ and integer $\ell\ge0$: $\E\,e^{-\theta t^2}=(1+2\theta)^{-1/2}$ and $\E\,t^{2\ell}e^{-\theta t^2}=(2\ell-1)!!\,(1+2\theta)^{-\ell-1/2}$, with $(-1)!!=1$. In particular $\E\,t^4e^{-t^2}=3^{-3/2}$ and $\E\,t^4=3$.
\item\label{F4} For symmetric $A\in\R^{d\times d}$: $\E_a(a^\top Aa)^2=\frac{(\mathrm{tr}A)^2+2\,\mathrm{tr}(A^2)}{d(d+2)}$.
\item\label{F5} $\E R^p=2^{p/2}\Gamma(D+p/2)/\Gamma(D)$ for $p>-d$; in particular $\E R^2=d$.
\end{enumerate}

We also use $H=\tfrac12(U-V)$, $C=\tfrac12(U+V)$, for which $U=C+H$, $V=C-H$, $C^\top H=0$, $\|C\|^2+\|H\|^2=1$, and $\E\|H\|^2=\Delta$.

\subsection{Proof of Lemma~\ref{lem:align} (alignment control)}

Let $Q=\|H\|^2=1-\|C\|^2\in[0,1]$, so $\E Q=\Delta$. Fix $t$. Conditioning on $(U,V)$ and applying \ref{F2},
\[
\E_a\,\varphi_Z(ta)=\E_{U,V}\,\E_{a,R}\,e^{itRa^\top C}=\E\,e^{-t^2\|C\|^2/2}=g(t)\,\E\,e^{t^2Q/2}\ \ge\ g(t)\,e^{t^2\Delta/2},
\]
by Jensen's inequality for the convex map $x\mapsto e^{t^2x/2}$. Hence $\E_a\varphi_Z(ta)-g(t)\ge g(t)(e^{t^2\Delta/2}-1)\ge0$. By Jensen's inequality for $w\mapsto|w|^2$ on $\mathbb C$,
\[
\E_a|\varphi_Z(ta)-g(t)|^2\ \ge\ \big|\E_a\varphi_Z(ta)-g(t)\big|^2\ \ge\ e^{-t^2}\big(e^{t^2\Delta/2}-1\big)^2=e^{-(1-\Delta)t^2}-2e^{-(1-\Delta/2)t^2}+e^{-t^2}.
\]
Taking $\E_t$ with \ref{F3} gives $\J\ge\E_te^{-t^2}(e^{t^2\Delta/2}-1)^2=(3-2\Delta)^{-1/2}-2(3-\Delta)^{-1/2}+3^{-1/2}=j(\Delta)$. Since $e^x-1\ge x$, $j(\Delta)\ge\frac{\Delta^2}{4}\E t^4e^{-t^2}=\frac{\Delta^2}{12\sqrt3}$ by \ref{F3}. Moreover $j'(\Delta)=(3-2\Delta)^{-3/2}-(3-\Delta)^{-3/2}>0$ on $(0,1]$, so $j$ is increasing, with $j(1)=1-\sqrt2+3^{-1/2}\approx0.163$. If $\J\le j(1)$, this gives $\Delta\le j^{-1}(\J)=\Delta^\star(\J)$; if $\J>j(1)$, then $\Delta\le1=\Delta^\star(\J)$. In both cases $\Delta^\star(\J)\le k_0\sqrt\J$, because $j(\Delta)\ge\Delta^2/k_0^2$ and $k_0\sqrt{j(1)}>1$.
On the family of \cref{sec:infonce-comparison}, $Q\equiv\Delta$, so the first Jensen step is an equality. Also $\varphi_Z(ta)=e^{-(1-\Delta)t^2/2}$ does not depend on $a$, so the second one is too. \hfill$\square$

\subsection{Proof of Proposition~\ref{prop:sg-minimizers}}

($\Rightarrow$) If $\J=0$, \cref{lem:align} gives $\Delta=0$, i.e., $\E\|U-V\|^2=0$, so $U=V$ a.s. The function $(a,t)\mapsto\varphi_Z(ta)-g(t)$ is continuous on $\Sd\times\R$, and its squared modulus integrates to zero against $\sigma\otimes\mathcal N(0,1)$, a measure with full support. Hence it vanishes identically. Since $\{ta\}=\R^d$, $\varphi_Z(\omega)=e^{-\|\omega\|^2/2}$ for all $\omega$, and $Z\sim\mathcal N(0,I_d)$ by uniqueness of characteristic functions. As $U=V$, $Z=RU$ with $R>0$ a.s., so $U=Z/\|Z\|\sim\sigma$ by \ref{F1}.
($\Leftarrow$) If $U=V\sim\sigma$, then $Z=RU\sim\mathcal N(0,I_d)$ by \ref{F1}, and $\J=0$. \hfill$\square$

\subsection{Proof of Proposition~\ref{prop:ideal} (directional targets)}

\begin{proof}\emph{(i)} Both terms are nonnegative and both weights are positive. A zero value therefore gives $I(X_1,X_2)=0$ and $S(X_j)=0$ for $j=1,2$. By the defining properties of $I$ and the characteristic population discrepancy $S$, $X_1=X_2$ almost surely and their common law is $\mathcal N(0,I_d)$. The converse is immediate. The Gaussian polar decomposition then gives aligned, uniform normalized directions.

\emph{(ii)} Let $X_1=X_2=X$ and draw $X$ uniformly from $\{\pm\sqrt d\,e_i:i\le d\}$. Then $\mathbb EX=0$, $\mathbb EXX^\top=I_d$, and the two views coincide. Hence the invariance, variance-hinge (target standard deviation one), and off-diagonal covariance terms all vanish. The direction $X/\|X\|$ has finite support and is not $\sigma$.
\end{proof}

\subsection{The centroid identity and the proof of Theorem~\ref{thm:structure}}\label{app:centroid}

\begin{lemma}[Centroid versus mixture]\label{lem:centroid}
For every paired law, with $T=T_{\bar q}$,
\[
\big\|\varphi_Z(t\,\cdot)-\varphi_T(t\,\cdot)\big\|\ \le\ \tfrac32\sqrt{\tfrac{d}{d+2}}\;\Delta\ \le\ \tfrac32\,\Delta.
\]
\end{lemma}

\begin{proof}
For every $x\in\R^d$, $\tfrac12(e^{ix^\top U}+e^{ix^\top V})=e^{ix^\top C}\cos(x^\top H)$. Put $x=tRa$ and take expectations over $(U,V,R)$. Since $T=RW$ with $W\sim\bar q=\frac12(q_U+q_V)$,
\[
\varphi_T(ta)=\tfrac12\big(\E e^{itRa^\top U}+\E e^{itRa^\top V}\big)=\E\big[e^{itRa^\top C}\cos(tRa^\top H)\big].
\]
Hence $\varphi_T(ta)-\varphi_Z(ta)=\E[e^{itRa^\top C}(\cos(tRa^\top H)-1)]$. Using $0\le1-\cos y\le y^2/2$, independence of $R$, and \ref{F5},
\[
|\varphi_T(ta)-\varphi_Z(ta)|\le\tfrac12t^2\,\E[R^2(a^\top H)^2]=\tfrac d2\,t^2\,a^\top\Sigma_Ha,\qquad \Sigma_H=\E\,HH^\top\succeq0,\ \mathrm{tr}\,\Sigma_H=\Delta .
\]
Squaring, integrating and using \ref{F3}--\ref{F4},
\[
\|\varphi_T-\varphi_Z\|^2\le\frac{d^2}{4}\cdot3\cdot\frac{(\mathrm{tr}\,\Sigma_H)^2+2\,\mathrm{tr}(\Sigma_H^2)}{d(d+2)}\le\frac{9d}{4(d+2)}\Delta^2,
\]
because $\mathrm{tr}(\Sigma_H^2)\le(\mathrm{tr}\,\Sigma_H)^2$ for $\Sigma_H\succeq0$.
\end{proof}

\begin{proof}[Proof of Theorem~\ref{thm:structure}]
By the reverse triangle inequality in $L^2$ and \cref{lem:centroid},
$|\sqrt{\J}-\Dlift(\bar q,\sigma)|=\big|\|\varphi_Z-g\|-\|\varphi_T-g\|\big|\le\|\varphi_Z-\varphi_T\|\le\frac32\Delta$.
\end{proof}

\subsection{InfoNCE: decomposition and uniqueness of the uniform minimizer}

For a probability measure $q$ on $\Sd$ let $g_q(u)=\int e^{\beta u^\top w}\,\dd q(w)\in[e^{-\beta},e^{\beta}]$ and $\Phi(q)=\int\log g_q\,\dd q$, so that $\Ub(q)=\Phi(q)-\log\kappa_d(\beta)$.

\begin{lemma}[Decomposition]\label{lem:decomp}
If $P$ has equal marginals $q$, then $\Lb(P)-\mathcal L^\star_\beta=2\beta\Delta+\Ub(q)$ and $\Ub(q)\le\log\big(1+\Db^2(q,\sigma)/\kappa_d(\beta)\big)$.
\end{lemma}

\begin{proof}
$-\beta\E U^\top V+\beta=\beta(1-\E U^\top V)=2\beta\Delta$, which gives the identity. Since $\int e^{\beta u^\top w}\dd\sigma(w)=\kappa_d(\beta)$ for every $u$,
\begin{equation}\label{eq:mmd-energy}
\Db^2(q,\sigma)=\iint k_\beta\,\dd q\,\dd q-2\iint k_\beta\,\dd q\,\dd\sigma+\iint k_\beta\,\dd\sigma\,\dd\sigma=\int g_q\,\dd q-\kappa_d(\beta).
\end{equation}
By Jensen's inequality for the concave logarithm, $\Phi(q)\le\log\int g_q\,\dd q=\log(\kappa_d(\beta)+\Db^2(q,\sigma))$.
\end{proof}

\citet[Thm.~1]{wang2020understanding} show that perfectly uniform encoders minimize $\Ub$; the following lemma gives a self-contained proof that $\sigma$ is its unique minimizer over all laws.

\begin{lemma}[Uniform is the unique minimizer]\label{lem:unif}
$\Phi$ is continuous for the weak topology, and $\Ub(q)\ge0$ for every $q$, with equality if and only if $q=\sigma$.
\end{lemma}

\begin{proof}
\emph{Continuity.} Since $|e^{\beta u^\top w}-e^{\beta u'^\top w}|\le\beta e^\beta\|u-u'\|$, the family $\{g_q\}$ is uniformly equicontinuous and uniformly bounded. If $q_n\rightharpoonup q$, then $g_{q_n}\to g_q$ pointwise and hence uniformly on the compact sphere (Arzel\`a--Ascoli). As $\log$ is Lipschitz on $[e^{-\beta},e^\beta]$,
$|\Phi(q_n)-\Phi(q)|\le e^{\beta}\|g_{q_n}-g_q\|_\infty+|\int\log g_q\,\dd(q_n-q)|\to0$.

\emph{Existence and structure of minimizers.} The set of probability measures on $\Sd$ is weakly compact, so $\Phi$ attains its minimum at some $q$. Let $m=\min_{u\in\mathrm{supp}\,q}g_q(u)$, which exists by continuity. Fix $\varepsilon>0$ and let $S_\varepsilon=\{u:g_q(u)<m+\varepsilon\}$, an open set with $q(S_\varepsilon)>0$. Let $q_S=q(\cdot\cap S_\varepsilon)/q(S_\varepsilon)$ and $q_\eta=(1-\eta)q+\eta q_S$ for $\eta\in[0,1]$. Then $g_{q_\eta}=g_q+\eta(g_{q_S}-g_q)$ takes values in $[e^{-\beta},e^\beta]$, and $\eta\mapsto\Phi(q_\eta)=\int\log g_{q_\eta}\,\dd q_\eta$ is differentiable, with right derivative at $0$
\[
\frac{\dd}{\dd\eta}\Phi(q_\eta)\Big|_{\eta=0^+}=\int\frac{g_{q_S}}{g_q}\,\dd q-1+\int\log g_q\,\dd q_S-\Phi(q).
\]
Minimality of $q$ forces this derivative to be $\ge0$. Now $g_q\ge m$ $q$-a.e., so by Fubini and the symmetry of $k_\beta$,
\[
\int\frac{g_{q_S}}{g_q}\,\dd q=\iint\frac{k_\beta(u,v)}{g_q(u)}\,\dd q(u)\,\dd q_S(v)\le\frac1m\int g_q\,\dd q_S\le\frac{m+\varepsilon}{m},
\]
and $\int\log g_q\,\dd q_S\le\log(m+\varepsilon)$. Therefore $\Phi(q)\le\frac{\varepsilon}{m}+\log(m+\varepsilon)$ for every $\varepsilon>0$, so $\Phi(q)\le\log m$. But $\Phi(q)=\int\log g_q\,\dd q\ge\log m$. Hence $\int(\log g_q-\log m)\,\dd q=0$ with a nonnegative integrand, and $g_q=m$ $q$-a.e. Integrating and using \eqref{eq:mmd-energy}, $m=\int g_q\,\dd q=\kappa_d(\beta)+\Db^2(q,\sigma)$, so
\[
\min\Phi=\Phi(q)=\log\big(\kappa_d(\beta)+\Db^2(q,\sigma)\big)\ge\log\kappa_d(\beta)=\Phi(\sigma),
\]
where $\Phi(\sigma)=\log\kappa_d(\beta)$ because $g_\sigma\equiv\kappa_d(\beta)$. Thus $\min\Phi=\log\kappa_d(\beta)$, i.e., $\Ub\ge0$. Moreover, every minimizer satisfies $\Db(q,\sigma)=0$, hence $q=\sigma$, because $k_\beta$ is characteristic on $\Sd$ (all its eigenvalues $\lambda^\beta_\ell$ with $\ell\ge1$ are positive, \cref{lem:eig-bounds}).
\end{proof}

\subsection{Spectral analysis of the two kernels}\label{app:spectral}

Let $\mathcal H_\ell$ be the space of degree-$\ell$ spherical harmonics on $\Sd$, of dimension $N_\ell=\frac{(2\ell+d-2)\,\Gamma(\ell+d-2)}{\Gamma(\ell+1)\Gamma(d-1)}$ ($N_0=1$), with an $L^2(\sigma)$-orthonormal basis $\{Y_{\ell m}\}_{m\le N_\ell}$. Let $P_\ell$ be the Gegenbauer polynomial normalized by $P_\ell(1)=1$ (for $d=2$, use $P_\ell(\cos\theta)=\cos(\ell\theta)$ and $N_\ell=2$ for $\ell\ge1$); then $|P_\ell|\le1$ on $[-1,1]$ and the addition theorem reads $\sum_mY_{\ell m}(u)Y_{\ell m}(v)=N_\ell P_\ell(u^\top v)$ \citep{dai2013approximation,atkinson2012spherical}. Put $\nu=D-1$.

\paragraph{Zonal kernels and MMD.} By Schoenberg's theorem \citep{schoenberg1942positive}, a continuous positive-definite zonal kernel has an expansion $K(u,v)=\sum_{\ell\ge0}\lambda_\ell N_\ell P_\ell(u^\top v)$ with $\lambda_\ell\ge0$ and $\sum_\ell\lambda_\ell N_\ell=K(u,u)<\infty$, converging absolutely and uniformly. Here $\lambda_\ell=\iint K(u,v)Y(u)Y(v)\,\dd\sigma\,\dd\sigma$ for any unit-norm $Y\in\mathcal H_\ell$. By uniform convergence, for every probability measure $q$,
\begin{equation}\label{eq:mmd-spectral}
\mathrm{MMD}_K^2(q,\sigma)=\sum_{\ell\ge1}\lambda_\ell\,\|\hat q_\ell\|^2,\qquad \|\hat q_\ell\|^2:=\sum_{m}\Big|\int Y_{\ell m}\,\dd q\Big|^2,
\end{equation}
where the $\ell=0$ term drops because $q-\sigma$ has zero mass. Since a finite measure on $\Sd$ is determined by its harmonic coefficients, $K$ is characteristic if and only if $\lambda_\ell>0$ for all $\ell\ge1$.

\paragraph{Funk--Hecke formulas.} For $Y\in\mathcal H_\ell$, $a\in\Sd$, $s\in\R$ and $\beta>0$ \citep[see, e.g.,][]{dai2013approximation,atkinson2012spherical},
\begin{align}
\int e^{is\,a^\top u}\,Y(u)\,\dd\sigma(u)&=i^\ell\,b_\ell(s)\,Y(a),\qquad b_\ell(s)=\Gamma(D)\sum_{k\ge0}\frac{(-1)^k(s/2)^{2k+\ell}}{k!\,\Gamma(k+\ell+D)},\label{eq:fh-plane}\\
\int e^{\beta\,a^\top u}\,Y(u)\,\dd\sigma(u)&=\lambda_\ell^\beta\,Y(a),\qquad\ \lambda^\beta_\ell=\Gamma(D)\sum_{k\ge0}\frac{(\beta/2)^{2k+\ell}}{k!\,\Gamma(k+\ell+D)}=\Gamma(D)\Big(\frac2\beta\Big)^{\nu}I_{\ell+\nu}(\beta).\label{eq:fh-exp}
\end{align}
For $s>0$, $b_\ell(s)=\Gamma(D)(2/s)^\nu J_{\ell+\nu}(s)$. In \eqref{eq:fh-plane}, $b_\ell$ is an entire function with parity $(-1)^\ell$, and $J$ denotes a Bessel function, unrelated to the loss $\J$. Formula \eqref{eq:fh-exp} identifies the eigenvalues of $k_\beta$, with $\lambda^\beta_0=\kappa_d(\beta)$.

\begin{lemma}[Lifted eigenvalues]\label{lem:lift-eig}
$K_{\rm lift}$ is a continuous positive-definite zonal kernel with $\mathrm{MMD}_{K_{\rm lift}}(q,\sigma)=\Dlift(q,\sigma)$. Its eigenvalues are $\lambda^{\rm lift}_\ell=\E_t\,h_\ell(t)^2$, where
\[
h_\ell(t)=\E_R\,b_\ell(tR)=c_\ell\,t^\ell\,\hypF\!\big(D+\tfrac\ell2;D+\ell;-\tfrac{t^2}2\big)=c_\ell\,t^\ell e^{-t^2/2}\,\hypF\!\big(\tfrac\ell2;D+\ell;\tfrac{t^2}2\big),\qquad c_\ell=\frac{\Gamma(D+\ell/2)}{2^{\ell/2}\,\Gamma(D+\ell)}.
\]
\end{lemma}

\begin{proof}
Continuity follows from dominated convergence and positive-definiteness from the feature-map form $K_{\rm lift}(u,v)=\langle\psi_u,\psi_v\rangle_{L^2}$. Zonality holds because $\psi_{Ou}(a,t)=\psi_u(O^\top a,t)$ for orthogonal $O$ and $\sigma$ is rotation invariant. Since $\varphi_{T_q}(ta)=\int\psi_u(a,t)\,\dd q(u)$ and $\varphi_{T_\sigma}(ta)=g(t)$, expanding the square gives $\Dlift^2(q,\sigma)=\iint K_{\rm lift}\,\dd(q-\sigma)^{\otimes2}$. For unit-norm $Y\in\mathcal H_\ell$, Fubini (all integrands are bounded) and \eqref{eq:fh-plane} give
\[
\int\psi_u(a,t)Y(u)\,\dd\sigma(u)=\E_R\int e^{itRa^\top u}Y(u)\,\dd\sigma(u)=i^\ell\,\E_R\,b_\ell(tR)\,Y(a)=i^\ell h_\ell(t)Y(a),
\]
so $\lambda^{\rm lift}_\ell=\E_{a,t}|h_\ell(t)Y(a)|^2=\E_t h_\ell(t)^2$. For the series, integrate \eqref{eq:fh-plane} termwise against the $\chi_d$ law using \ref{F5}. This is legitimate because $\sum_k\E R^{2k+\ell}(|t|/2)^{2k+\ell}/(k!\Gamma(k+\ell+D))<\infty$, the ratio of consecutive terms tending to zero. It gives
\[
h_\ell(t)=\frac{t^\ell}{2^{\ell/2}}\sum_{k\ge0}\frac{\Gamma(D+\ell/2+k)}{k!\,\Gamma(D+\ell+k)}\Big(-\frac{t^2}{2}\Big)^k=c_\ell\,t^\ell\,\hypF\!\big(D+\tfrac\ell2;D+\ell;-\tfrac{t^2}2\big).
\]
The last expression follows from Kummer's transformation $\hypF(a;b;-z)=e^{-z}\hypF(b-a;b;z)$ \citep[Eq.~13.2.39]{dlmf}.
\end{proof}

\begin{lemma}[Eigenvalue bounds]\label{lem:eig-bounds}
For all $\ell\ge1$, $d\ge2$ and $\beta>0$:
\begin{enumerate}[label=(\roman*),leftmargin=*]
\item $\displaystyle\frac{\Gamma(D)}{\Gamma(D+\ell)}\Big(\frac\beta2\Big)^\ell\ \le\ \lambda^\beta_\ell\ \le\ \frac{\Gamma(D)}{\Gamma(D+\ell)}\Big(\frac\beta2\Big)^\ell\exp\!\Big(\min\Big\{\beta,\frac{\beta^2}{4(D+\ell)}\Big\}\Big)$;
\item $c_\ell^2\,(2\ell-1)!!\,3^{-\ell-1/2}\ \le\ \lambda^{\rm lift}_\ell\ \le\ c_\ell^2\,(2\ell-1)!!$;
\item $A(D,\ell):=\frac{\Gamma(D)\Gamma(D+\ell)}{\Gamma(D+\ell/2)^2}$ satisfies $1\le A(D,\ell)\le A(1,\ell)\le\sqrt2\,(2\ell-1)!!/\ell!$.
\end{enumerate}
\end{lemma}

\begin{proof}
(i) The lower bound is the $k=0$ term of the positive series \eqref{eq:fh-exp}. For the first upper bound, use $I_\mu(\beta)=\frac{(\beta/2)^\mu}{\sqrt\pi\,\Gamma(\mu+1/2)}\int_{-1}^1e^{\beta x}(1-x^2)^{\mu-1/2}\dd x$ for $\mu>-\tfrac12$ \citep[Eq.~10.32.2]{dlmf}, with $\mu=\ell+\nu\ge\ell\ge1$. Bounding $e^{\beta x}\le e^\beta$ and using $\int_{-1}^1(1-x^2)^{\mu-1/2}\dd x=\sqrt\pi\,\Gamma(\mu+\frac12)/\Gamma(\mu+1)$ gives $I_\mu(\beta)\le(\beta/2)^\mu e^\beta/\Gamma(\mu+1)$. For the second, $\Gamma(k+\ell+D)=\Gamma(\ell+D)(\ell+D)_k\ge\Gamma(\ell+D)(\ell+D)^k$ in \eqref{eq:fh-exp}.

(ii) Lower bound: in the Kummer form of \cref{lem:lift-eig}, $\hypF(\ell/2;D+\ell;t^2/2)\ge1$ since all its terms are nonnegative, so $h_\ell(t)^2\ge c_\ell^2t^{2\ell}e^{-t^2}$; apply \ref{F3} with $\theta=1$. Upper bound: $|J_\mu(s)|\le(s/2)^\mu/\Gamma(\mu+1)$ for $s\ge0$ and $\mu\ge-\frac12$ \citep[Eq.~10.14.4]{dlmf} gives $|b_\ell(s)|\le\Gamma(D)(s/2)^\ell/\Gamma(D+\ell)$ for $s\ge0$, and by the parity of $b_\ell$ also $|b_\ell(s)|\le\Gamma(D)|s/2|^\ell/\Gamma(D+\ell)$ for all real $s$. Hence, by \ref{F5}, $|h_\ell(t)|\le\Gamma(D)(|t|/2)^\ell\E R^\ell/\Gamma(D+\ell)=c_\ell|t|^\ell$, and $\lambda^{\rm lift}_\ell\le c_\ell^2\E t^{2\ell}=c_\ell^2(2\ell-1)!!$.

(iii) $A\ge1$ is log-convexity of $\Gamma$. Next, $\partial_D\log A(D,\ell)=\psi(D)+\psi(D+\ell)-2\psi(D+\ell/2)\le0$ because the digamma function $\psi$ is concave on $(0,\infty)$. Hence $A(D,\ell)\le A(1,\ell)=\ell!/\Gamma(1+\ell/2)^2$ for $D\ge1$. Let $x=(\ell+1)/2$. The duplication formula $\Gamma(2z)=\pi^{-1/2}2^{2z-1}\Gamma(z)\Gamma(z+\frac12)$ \citep[Eq.~5.5.5]{dlmf}, applied with $2z=\ell+1$ and with $2z=\ell+\frac12$, gives $\ell!/\Gamma(1+\ell/2)=2^\ell\Gamma(x)/\sqrt\pi$ and $(2\ell-1)!!=2^\ell\Gamma(\ell+\frac12)/\sqrt\pi=2^{2\ell-1/2}\Gamma(x-\frac14)\Gamma(x+\frac14)/\pi$. Therefore
\[
\frac{\ell!\,A(1,\ell)}{(2\ell-1)!!}=\sqrt2\,\frac{\Gamma(x)^2}{\Gamma(x-\frac14)\Gamma(x+\frac14)}\le\sqrt2,
\]
again by log-convexity of $\Gamma$.
\end{proof}

\subsection{Proof of Theorem~\ref{thm:spectral} and Corollary~\ref{cor:examples}}\label{app:spectral-proofs}

\begin{lemma}[Eigenvalues of $k_\beta$ lie in $(0,\kappa_d(\beta))$]\label{lem:lam-range}
For $\ell\ge1$, $0<\lambda^\beta_\ell<\kappa_d(\beta)=\lambda^\beta_0$.
\end{lemma}
\begin{proof}
Positivity follows from the positive series \eqref{eq:fh-exp}. By the Funk--Hecke formula, $\lambda^\beta_\ell=\E[e^{\beta W_1}P_\ell(W_1)]$ and $\kappa_d(\beta)=\E[e^{\beta W_1}]$ for $W\sim\sigma$. Since $|P_\ell|\le1$, with $|P_\ell(x)|<1$ for all but finitely many $x\in(-1,1)$ when $\ell\ge1$, and $W_1$ has a density, the inequality is strict.
\end{proof}

\begin{proof}[Proof of Theorem~\ref{thm:spectral}]
\emph{Global bound.} $0\le\Ub\le\log(1+\Db^2/\kappa)$ follows from \cref{lem:decomp,lem:unif}.

\emph{(i)} Let $\kappa=\kappa_d(\beta)$, $u=\dd q/\dd\sigma-1$ (so $\int u\,\dd\sigma=0$ and $\|u\|_\infty=\delta$), and let $\mathcal K$ be the integral operator of $k_\beta$ on $L^2(\sigma)$, which acts on degree $\ell$ as multiplication by $\lambda_\ell:=\lambda^\beta_\ell$. Then $g_q=\kappa+\mathcal Ku$, and $h:=\mathcal Ku/\kappa$ satisfies $\int h\,\dd\sigma=0$ and $\|h\|_\infty\le\delta\le\frac12$, because $|\mathcal Ku(v)|\le\delta\int e^{\beta v^\top w}\dd\sigma(w)=\delta\kappa$ for every $v$. For $|x|\le\frac12$, $\log(1+x)=x-\frac12x^2+r(x)$ with $|r(x)|\le\frac{|x|^3}{3(1-|x|)}\le\frac23|x|^3$. Hence
\[
\Ub(q)=\int\log(1+h)\,(1+u)\,\dd\sigma=\underbrace{\int hu\,\dd\sigma-\tfrac12\int h^2\,\dd\sigma}_{=\mathcal Q_\beta(q)}\ \underbrace{-\ \tfrac12\int h^2u\,\dd\sigma+\int r(h)(1+u)\,\dd\sigma}_{=:\mathcal R}.
\]
By orthogonality of harmonics, $\int hu=\sum_\ell\lambda_\ell\|u_\ell\|^2/\kappa$ and $\int h^2=\sum_\ell\lambda_\ell^2\|u_\ell\|^2/\kappa^2\le\Db^2/\kappa$, using $\lambda_\ell<\kappa$ (\cref{lem:lam-range}) and $u_\ell=\hat q_\ell$ for $\ell\ge1$. Therefore $|\mathcal R|\le\big(\frac\delta2+\frac23(1+\delta)\|h\|_\infty\big)\int h^2\le\big(\frac12+\frac23\cdot\frac32\big)\delta\,\Db^2/\kappa=\frac32\,\delta\,\Db^2/\kappa$. Finally, each weight $\frac{\lambda_\ell}{\kappa}(1-\frac{\lambda_\ell}{2\kappa})$ lies in $[\frac{\lambda_\ell}{2\kappa},\frac{\lambda_\ell}{\kappa}]$, and \eqref{eq:mmd-spectral} gives the bounds on $\mathcal Q_\beta$.

\emph{(ii)} Upper bound: by the global bound, $\log(1+x)\le x$ and \eqref{eq:mmd-spectral},
$\Ub(q)\le\Db^2/\kappa=\kappa^{-1}\sum_{\ell\ge1}\lambda_\ell\|\hat q_\ell\|^2\le(S_K/\kappa)\sum_{\ell\ge1}\mu_\ell\|\hat q_\ell\|^2=(S_K/\kappa)\,\mathrm{MMD}_K^2$.
Lower bound: fix $\ell\ge1$ and $e\in\Sd$, and apply (i) to $q_\varepsilon=(1+\varepsilon P_\ell(e^\top\cdot))\sigma$, for which $\delta\le\varepsilon$ and only the degree-$\ell$ component of $q_\varepsilon-\sigma$ is nonzero. If $\mu_\ell=0$, then $\mathrm{MMD}_K(q_\varepsilon,\sigma)=0<\Ub(q_\varepsilon)$ for small $\varepsilon>0$ by uniqueness (\cref{lem:unif}), so $c_K=\infty=S_K$. Otherwise
\[
\frac{\Ub(q_\varepsilon)}{\mathrm{MMD}_K^2(q_\varepsilon,\sigma)}=\frac{\lambda_\ell(1-\lambda_\ell/(2\kappa))}{\kappa\,\mu_\ell}+O(\varepsilon)\ \ge\ \frac{\lambda_\ell}{2\kappa\,\mu_\ell}+O(\varepsilon),
\]
using $\lambda_\ell<\kappa$ (\cref{lem:lam-range}). Let $\varepsilon\to0$ and take the supremum over $\ell$.
\end{proof}

\begin{proof}[Proof of Corollary~\ref{cor:examples}]
(a) is \cref{prop:kernel}. (b) Since $\|u-v\|^2=2-2u^\top v$, we have $e^{-\gamma\|u-v\|^2}=e^{-2\gamma}k_{2\gamma}(u,v)$, so $\mu_\ell=e^{-2\gamma}\lambda^{2\gamma}_\ell$. By \eqref{eq:fh-exp}, $\lambda^x_\ell=\Gamma(D)(x/2)^\ell S_\ell(x)$ with $S_\ell(x)=\sum_k(x/2)^{2k}/(k!\Gamma(k+\ell+D))$, which is increasing in $x>0$. If $2\gamma\ge\beta$, then $\lambda^\beta_\ell/\mu_\ell=e^{2\gamma}(\beta/2\gamma)^\ell S_\ell(\beta)/S_\ell(2\gamma)\le e^{2\gamma}$. At $2\gamma=\beta$, $\mu_\ell=e^{-\beta}\lambda^\beta_\ell$ for every $\ell$, so $S_K=e^\beta$; by \eqref{eq:fh-exp}, $\kappa_d(\beta)=\Gamma(D)(2/\beta)^{D-1}I_{D-1}(\beta)$, and $I_\nu(\beta)=e^\beta(2\pi\beta)^{-1/2}(1+O(\beta^{-1}))$ \citep[Eq.~10.40.1]{dlmf} gives $e^\beta/\kappa_d(\beta)=O(\beta^{(d-1)/2})$. If $2\gamma<\beta$, the same ratio is at least $e^{2\gamma}(\beta/2\gamma)^\ell\to\infty$. For the uniformity loss, $\log\E_{q\otimes q}e^{-\gamma\|u-v\|^2}=-2\gamma+\log\iint k_{2\gamma}\,\dd q\,\dd q=-2\gamma+\log(\kappa_d(2\gamma)+D_{2\gamma}^2(q,\sigma))$ by \eqref{eq:mmd-energy}. Hence $D_{2\gamma}^2=\kappa_d(2\gamma)(e^{\Delta\mathcal L}-1)$, where $\Delta\mathcal L$ is the excess over $\sigma$. When $2\gamma\ge\beta$, $\lambda^\beta_\ell\le\lambda^{2\gamma}_\ell$ for all $\ell\ge1$, so $\Ub\le D_{2\gamma}^2/\kappa_d(\beta)$; with $\iint k_{2\gamma}\,\dd q\,\dd q\le e^{2\gamma}$ this gives $D_{2\gamma}^2=\kappa_d(2\gamma)(e^{\Delta\mathcal L}-1)\le e^{2\gamma}\Delta\mathcal L$ and $\Ub\le e^{2\gamma}\Delta\mathcal L/\kappa_d(\beta)$. Conversely, $\Delta\mathcal L\le D^2_{2\gamma}/\kappa_d(2\gamma)$, so linear control of $\Ub$ by $\Delta\mathcal L$ fails whenever $2\gamma<\beta$. With $\gamma=2$, linear control holds exactly when $\tau=1/\beta\ge1/4$. (c) For the heat kernel, $\mu_\ell=e^{-s\ell(\ell+d-2)}$ decays faster than $\lambda^\beta_\ell\ge\Gamma(D)(\beta/2)^\ell/\Gamma(D+\ell)$ (\cref{lem:eig-bounds}), so $S_K=\infty$. A band-limited kernel has $\mu_\ell=0$ for some $\ell\ge1$. First- and second-moment matching is $\mathrm{MMD}_K$ for $K(u,v)=u^\top v+(u^\top v)^2$, whose eigenvalues vanish for $\ell\ge3$.
\end{proof}

\subsection{Proof of Proposition~\ref{prop:kernel} (lifted kernel)}\label{app:kernel-proof}

\emph{Comparison and optimality.} By \eqref{eq:mmd-spectral} for both kernels, $\Db^2=\sum_{\ell\ge1}\lambda^\beta_\ell\|\hat q_\ell\|^2\le C^\star_{d,\beta}\sum_{\ell\ge1}\lambda^{\rm lift}_\ell\|\hat q_\ell\|^2=C^\star_{d,\beta}\Dlift^2$. Note that $\lambda^{\rm lift}_\ell>0$ by \cref{lem:eig-bounds}(ii). Conversely, fix $e\in\Sd$, $\ell\ge1$, and let $q=(1+\frac12P_\ell(e^\top u))\,\sigma(\dd u)$, a probability measure because $|P_\ell|\le1$ and $\int P_\ell(e^\top u)\dd\sigma=0$. Then $q-\sigma$ has only a degree-$\ell$ component, so $\Db^2/\Dlift^2=\lambda^\beta_\ell/\lambda^{\rm lift}_\ell$. Hence no constant smaller than $C^\star_{d,\beta}$ works.

\emph{Upper bound.} Writing $c_\ell^2=\Gamma(D)/(2^\ell A(D,\ell)\Gamma(D+\ell))$ and combining \cref{lem:eig-bounds}(i)--(iii),
\[
\frac{\lambda_\ell^\beta}{\lambda_\ell^{\rm lift}}\le\frac{\Gamma(D)(\beta/2)^\ell e^\beta}{\Gamma(D+\ell)}\cdot\frac{2^\ell A(D,\ell)\Gamma(D+\ell)\,3^{\ell+1/2}}{\Gamma(D)\,(2\ell-1)!!}=\sqrt3\,e^\beta(3\beta)^\ell\frac{A(D,\ell)}{(2\ell-1)!!}\le\sqrt6\,e^\beta\frac{(3\beta)^\ell}{\ell!}\le\sqrt6\,e^{4\beta}.
\]

\emph{Growth of $C^\star_{d,\beta}/\kappa_d(\beta)$.} The constant that enters \eqref{eq:main} is $C^\star_{d,\beta}/\kappa_d(\beta)$. For its lower bound, fix $\ell\ge1$. By \cref{lem:eig-bounds}(ii), $C^\star_{d,\beta}\ge\lambda^\beta_\ell/(c_\ell^2(2\ell-1)!!)$, and by \eqref{eq:fh-exp},
\[
\frac{\lambda^\beta_\ell}{\kappa_d(\beta)}=\frac{I_{\nu+\ell}(\beta)}{I_\nu(\beta)}=\prod_{k=0}^{\ell-1}r_{\nu+k}(\beta),\qquad r_\mu:=\frac{I_{\mu+1}}{I_\mu}\ \ge\ \frac{\beta}{\mu+1+\sqrt{\beta^2+(\mu+1)^2}}=e^{-\operatorname{asinh}((\mu+1)/\beta)},
\]
where the inequality is Amos' bound, valid for $\mu\ge0$ \citep{amos1974computation}; here $\nu=D-1\ge0$. Since $\operatorname{asinh}$ is increasing, $\sum_{k=0}^{\ell-1}\operatorname{asinh}\frac{\nu+k+1}\beta\le\beta\int_0^{(\nu+\ell+1)/\beta}\operatorname{asinh}(x)\,\dd x$. On the other side, $c_\ell^2(2\ell-1)!!=\Gamma(D+\frac\ell2)^2\,\Gamma(\ell+\frac12)/(\sqrt\pi\,\Gamma(D+\ell)^2)$, and Stirling's formula gives $\frac1\ell\log\big(c_\ell^2(2\ell-1)!!\big)\to-\log2$ as $\ell\to\infty$ for fixed $d$. Take $\ell=\lfloor c\beta\rfloor$ and let $\beta\to\infty$:
\[
\liminf_{\beta\to\infty}\frac1\beta\log\frac{C^\star_{d,\beta}}{\kappa_d(\beta)}\ \ge\ c\log2-\int_0^c\operatorname{asinh}(x)\,\dd x=c\log2-c\operatorname{asinh}c+\sqrt{1+c^2}-1 .
\]
The right-hand side is maximized where $\operatorname{asinh}c=\log2$, i.e., $c=\frac34$, where it equals $\frac54-1=\frac14$. This proves the last claim of \cref{prop:kernel}. By \cref{thm:spectral}(ii), $C^\star_{d,\beta}/\kappa_d(\beta)$ is optimal up to a factor of $2$ among bounds of the form $\Ub\le c\,\Dlift^2$, so the exponential growth is intrinsic to the lifted kernel. \hfill$\square$

\subsection{Proof of Theorem~\ref{thm:infonce} (sharp comparison with population InfoNCE)}\label{app:infonce-proof}

\emph{Upper bound.} With equal marginals, $\bar q=q$. By \cref{lem:decomp}, $\Lb-\mathcal L^\star_\beta\le2\beta\Delta+\log(1+\Db^2(q,\sigma)/\kappa_d(\beta))$. By \cref{lem:align}, $\Delta\le k_0\sqrt\J$. By \cref{prop:kernel,thm:structure,lem:align},
\[
\Db^2(q,\sigma)\le C^\star_{d,\beta}\Dlift^2(q,\sigma)\le C^\star_{d,\beta}\big(\sqrt\J+\tfrac32\Delta\big)^2\le C^\star_{d,\beta}\big(1+\tfrac32k_0\big)^2\J.
\]
This gives \eqref{eq:main}.
The lower bound $\Lb\ge\mathcal L^\star_\beta$ is \cref{lem:decomp,lem:unif}.

\emph{Minimizing sequences.} ``$\Rightarrow$'' follows from \eqref{eq:main}. For ``$\Leftarrow$'', suppose $\Lb(P_n)-\mathcal L^\star_\beta\to0$. By \cref{lem:decomp,lem:unif}, both nonnegative terms $2\beta\Delta_n$ and $\Ub(q_n)$ tend to zero. We claim $q_n\rightharpoonup\sigma$. Otherwise, by compactness, some subsequence converges weakly to $q'\ne\sigma$, and continuity of $\Phi$ (\cref{lem:unif}) gives $\Ub(q')=0$, which contradicts uniqueness. Since $K_{\rm lift}$ is bounded and continuous, $q_n\otimes q_n\rightharpoonup\sigma\otimes\sigma$ implies $\Dlift^2(q_n,\sigma)=\iint K_{\rm lift}\,\dd(q_n-\sigma)^{\otimes2}\to0$. By \cref{thm:structure}, $\sqrt{\J(P_n)}\le\Dlift(q_n,\sigma)+\frac32\Delta_n\to0$.

\emph{Refined alignment term and sharpness.} In the upper bound we may use $\Delta\le\Delta^\star(\J)$ (\cref{lem:align}) in place of $\Delta\le k_0\sqrt\J$, which gives the refined form stated after \cref{thm:infonce}. Consider the family of \cref{sec:infonce-comparison}. Let $(n,h)$ be a Haar-distributed orthonormal pair and, for $s\in[0,1]$, let $U=\sqrt s\,n+\sqrt{1-s}\,h$ and $V=\sqrt s\,n-\sqrt{1-s}\,h$. Both are unit vectors, and for every orthogonal $O$ the pair $(On,Oh)$ has the law of $(n,h)$; hence $U$ and $V$ are both $\sigma$-distributed. Moreover $C=\sqrt s\,n$, $\Delta=1-s$, and $Z=\sqrt s\,Rn\sim\mathcal N(0,sI_d)$ by \ref{F1}, so by \ref{F3}
\begin{equation}\label{eq:Js}
\J=\E_t\big(e^{-st^2/2}-e^{-t^2/2}\big)^2=(1+2s)^{-1/2}-2(2+s)^{-1/2}+3^{-1/2}=j(1-s)=j(\Delta).
\end{equation}
Since $q=\sigma$, $\Ub(q)=0$ and $\Lb-\mathcal L^\star_\beta=2\beta\Delta=2\beta\Delta^\star(\J)$ exactly. For the limit, the upper bound gives $(\Lb-\mathcal L^\star_\beta)/\sqrt\J\le2\beta k_0+c\sqrt\varepsilon$ when $0<\J\le\varepsilon$, with $c=\bar c_{d,\beta}$. Along the family, $j(\Delta)=\Delta^2/(12\sqrt3)+O(\Delta^3)$ (indeed $j(0)=j'(0)=0$ and $j''(0)=\frac32 3^{-5/2}$), so the ratio tends to $2\beta k_0$ as $\Delta\downarrow0$, with $\J>0$ for $\Delta>0$. \hfill$\square$

\subsection{Proof of Proposition~\ref{thm:barrier} (square-root barrier)}

Write $F(P)=\|\mathcal A(P)-m^\star\|^2_{\mathcal H}$ with $\mathcal A(P)=\E_P\Psi(U,V,\xi)$, which is affine in $P$: $\mathcal A(P_p)=(1-p)\mathcal A(P_0)+p\mathcal A(P_1)$. Since $F(P_0)=0$, $\mathcal A(P_0)=m^\star$, so $\mathcal A(P_p)-m^\star=p(\mathcal A(P_1)-m^\star)$ and $F(P_p)=p^2F(P_1)$. Both marginals of $P_0$ and of $P_1$ are $\sigma$, hence so are those of $P_p$, and $\Ub(\sigma)=0$. Since $\Delta$ is linear in the law and $\Delta(P_0)=0$, \cref{lem:decomp} gives $\Lb(P_p)-\mathcal L^\star_\beta=2\beta\Delta(P_p)=2\beta p\Delta(P_1)$.

For the consequence, let $P_1$ be any law with uniform marginals and $\Delta(P_1)>0$ (for instance $\sigma\otimes\sigma$, or a member of the family of \cref{sec:infonce-comparison}). If $F(P_1)=0$, then $F$ vanishes on a law with $\Lb>\mathcal L^\star_\beta$. Otherwise, along $P_p$ we have $\Lb(P_p)-\mathcal L^\star_\beta=c\sqrt{F(P_p)}$ with $c=2\beta\Delta(P_1)/\sqrt{F(P_1)}>0$ and $F(P_p)\downarrow0$ as $p\downarrow0$, which contradicts $\Lb-\mathcal L^\star_\beta\le\psi(F)$ with $\psi(x)/\sqrt x\to0$. \emph{Coverage.} $\J$ is Hilbertian with $\mathcal H=L^2(\sigma\otimes\mathcal N(0,1))$, $\xi=R$, $\Psi(U,V,R)=\big[(a,t)\mapsto e^{itRa^\top C}\big]$ (of norm $1$), and $m^\star=g$; changing the radius law or the frequency weight changes only $\Psi$ and $\mathcal H$. Every squared MMD with a bounded kernel is Hilbertian through its kernel mean embedding. \hfill$\square$

\subsection{Proof of Theorem~\ref{thm:slaving} (nuisance preference of pure SG) and Corollary~\ref{cor:threshold}}\label{app:slaving}

\paragraph{Setting.} Let $(\mathsf P_\vartheta)_{\vartheta\ge0}=(e^{\vartheta\nabla^2_{\mathbb S}})$ be the heat semigroup of the Laplace--Beltrami operator $\nabla^2_{\mathbb S}$ on $\Sd$, with kernel $p_\vartheta(w,\cdot)$. Given $W\sim q$, the views $U,V$ are drawn independently from $p_\vartheta(W,\cdot)$; write $P_\vartheta^q$ for the resulting paired law. Since $\nabla^2_{\mathbb S}$ acts on $\mathcal H_\ell$ as $-\ell(\ell+d-2)$, $\E[U\mid W]=e^{-(d-1)\vartheta}W$, so $\E U^\top V=e^{-2(d-1)\vartheta}$ and $\Delta(\vartheta)=\frac12(1-e^{-2(d-1)\vartheta})$, independently of $q$. The marginal of each view is $\mathsf P_\vartheta q$, with harmonic coefficients $e^{-\ell(\ell+d-2)\vartheta}\hat q_\ell$.

\paragraph{Step 1: the pair generator.} For a smooth $F:\R^d\to\mathbb C$ let $G(u,v)=F(\frac{u+v}2)$. For a function $\phi$ on $\R^d$ restricted to the sphere, $\nabla^2_{\mathbb S}\phi(u)=\mathrm{tr}(\Pi_u\nabla^2\phi(u))-(d-1)u^\top\nabla\phi(u)$ with $\Pi_u=I-uu^\top$. Applying this in $u$ and in $v$ and evaluating at $u=v=w$ gives
\[
(\nabla^2_{\mathbb S,u}+\nabla^2_{\mathbb S,v})G\big|_{u=v=w}=\tfrac12\mathrm{tr}(\Pi_w\nabla^2F(w))-(d-1)\,w^\top\nabla F(w)=\tfrac12\nabla^2_{\mathbb S}F(w)-\tfrac{d-1}2\,w^\top\nabla F(w)=:\mathcal GF(w).
\]
The first term is the heat generator at half speed; the second is the generator of the radial contraction $c\mapsto e^{-(d-1)\vartheta/2}c$. Since $G$ is smooth on $\Sd\times\Sd$ and the product semigroup $\mathsf P_\vartheta\otimes\mathsf P_\vartheta$ has generator $\nabla^2_{\mathbb S,u}+\nabla^2_{\mathbb S,v}$,
\begin{equation}\label{eq:pairgen}
\frac{\dd}{\dd\vartheta}\E\,F(C)\Big|_{\vartheta=0}=\int\mathcal GF\,\dd q,\qquad C=\tfrac12(U+V).
\end{equation}
We apply this to $F_{a,t}(c)=\E_Re^{itRa^\top c}$, for which $c^\top\nabla F_{a,t}(c)=t\,\partial_tF_{a,t}(c)$. All derivatives of $F_{a,t}$ are bounded by polynomials in $|t|$ (moments of $R$, \ref{F5}). Since $\frac{\dd^k}{\dd\vartheta^k}\E F_{a,t}(C)=\E[(\mathcal L_u+\mathcal L_v)^kG(U,V)]$ with $\mathcal L=\nabla^2_{\mathbb S}$, the $\vartheta$-derivatives of every order are bounded by polynomials in $|t|$ uniformly in $\vartheta\in[0,1]$. This justifies differentiating under $\E_{a,t}$ below by dominated convergence and shows that $\vartheta\mapsto\J$ is $C^\infty$ on $[0,1]$.

\paragraph{Step 2: part (i).} Write $\varphi_\vartheta(ta)=\E F_{a,t}(C)$ under $P^q_\vartheta$ and $u=\dd q/\dd\sigma-1=\sum_{\ell\ge1}u_\ell$ (for a general probability measure $q$, read $\int\cdot\,u\,\dd\sigma$ as integration against $q-\sigma$ and $u_\ell$ as the degree-$\ell$ projection of $q-\sigma$; the expansions below converge because $F_{a,t}$ and $\mathcal GF_{a,t}$ are smooth on the sphere). At $\vartheta=0$, $C=W$, and by \cref{lem:lift-eig}
\[
D(a,t):=\varphi_0(ta)-g(t)=\int F_{a,t}\,\dd(q-\sigma)=\sum_{\ell\ge1}i^\ell h_\ell(t)\,u_\ell(a).
\]
By \eqref{eq:pairgen}, $\partial_\vartheta\varphi_\vartheta|_0=\int\mathcal GF_{a,t}\,\dd\sigma+\int\mathcal GF_{a,t}\,u\,\dd\sigma$. The first integral equals $-\frac{d-1}2t\partial_t\int F_{a,t}\dd\sigma=\frac{d-1}2t^2g(t)$, a function of $t$ alone. For the second, self-adjointness of $\nabla^2_{\mathbb S}$ and $\int F_{a,t}u\,\dd\sigma=D(a,t)$ give
\[
\int\mathcal GF_{a,t}\,u\,\dd\sigma=-\tfrac12\sum_{\ell\ge1}\ell(\ell+d-2)\,i^\ell h_\ell(t)u_\ell(a)-\tfrac{d-1}2\sum_{\ell\ge1}i^\ell\,t\,h_\ell'(t)\,u_\ell(a).
\]
Now $\partial_\vartheta\J|_0=2\,\mathrm{Re}\,\E_{a,t}[\overline{D}\,\partial_\vartheta\varphi_\vartheta|_0]$. The contribution of $\frac{d-1}2t^2g(t)$ vanishes, because $\E_aD(a,t)=0$ (every $u_\ell$ with $\ell\ge1$ has mean zero under $\sigma$). By orthogonality of harmonics of different degrees, and since $h_\ell$ is real,
\[
\partial_\vartheta\J|_{\vartheta=0}=-\sum_{\ell\ge1}\Big(\ell(\ell+d-2)\E_th_\ell^2+(d-1)\E_t[t\,h_\ell h_\ell']\Big)\|u_\ell\|^2 .
\]
Gaussian integration by parts, $\E[t\,\phi'(t)]=\E[(t^2-1)\phi(t)]$ for $t\sim\mathcal N(0,1)$, applied to $\phi=h_\ell^2$ gives $(d-1)\E[th_\ell h_\ell']=\frac{d-1}2\E[(t^2-1)h_\ell^2]$, which proves part (i) with $G_\ell=\ell(\ell+d-2)\lambda^{\rm lift}_\ell+\frac{d-1}2\E_t[(t^2-1)h_\ell(t)^2]$ and $\|u_\ell\|^2=\|\hat q_\ell\|^2$. For the bounds on $G_\ell$ we show
\begin{equation}\label{eq:radial}
\frac{2\ell+1}{3}\ <\ \frac{\E_t[t^2h_\ell(t)^2]}{\E_t[h_\ell(t)^2]}\ <\ 2\ell+1 .
\end{equation}
By Kummer's Euler integral, $\hypF(a;b;-z)=\E\,e^{-zS}$ with $S\sim\mathrm{Beta}(a,b-a)$ for $b>a>0$ \citep[Eq.~13.4.1]{dlmf}, so \cref{lem:lift-eig} gives $h_\ell(t)=c_\ell t^\ell\,\E e^{-St^2/2}$ with $S\sim\mathrm{Beta}(D+\frac\ell2,\frac\ell2)$. Let $S'$ be an independent copy and $A=1+S+S'\in(1,3)$ almost surely. By Fubini and \ref{F3},
\[
\E_t[t^{2m}h_\ell^2]=c_\ell^2\,\E_{S,S'}\E_t\big[t^{2\ell+2m}e^{-(S+S')t^2/2}\big]=c_\ell^2\,(2\ell+2m-1)!!\,\E A^{-\ell-m-1/2},
\]
so the ratio in \eqref{eq:radial} equals $(2\ell+1)\,\E A^{-\ell-3/2}/\E A^{-\ell-1/2}$, and $1/3<A^{-1}<1$ gives \eqref{eq:radial}. Hence $\frac{2(\ell-1)}3\lambda^{\rm lift}_\ell<\E[(t^2-1)h_\ell^2]<2\ell\lambda^{\rm lift}_\ell$, which, combined with the expression for $G_\ell$ above, gives
\begin{equation}\label{eq:G-bracket}
\tfrac{(d-1)(\ell-1)}3<\frac{G_\ell}{\lambda^{\rm lift}_\ell}-\ell(\ell+d-2)<(d-1)\ell ;
\end{equation}
in particular the radial term is strictly positive for every $\ell\ge1$. Finally $\lambda^{\rm lift}_\ell>0$ (\cref{lem:eig-bounds}(ii)).

\paragraph{Step 3: part (ii).} The map $q\mapsto P^q_\vartheta$ is linear, so for $q_\varepsilon=(1+\varepsilon f)\sigma$ we have $\varphi_\vartheta-g=A_\vartheta+\varepsilon B_\vartheta$, where $A_\vartheta(t)=\varphi_\vartheta^{\sigma}(ta)-g(t)$ depends on $t$ only (rotation invariance) and $B_\vartheta(a,t)=\int\E[F_{a,t}(C)\mid W=w]f(w)\,\dd\sigma(w)$. Rotation equivariance of the channel makes $\E_a\E[F_{a,t}(C)\mid W=w]$ independent of $w$, so $\E_aB_\vartheta(a,t)=0$ and the cross term vanishes:
\[
\J(\vartheta,\varepsilon)=\mathcal A(\vartheta)+\varepsilon^2\mathcal C(\vartheta),\qquad\mathcal A(\vartheta)=\E_tA_\vartheta^2,\quad\mathcal C(\vartheta)=\E_{a,t}|B_\vartheta|^2 .
\]
Here $\mathcal A,\mathcal C\in C^\infty[0,1]$ by Step 1; $\mathcal C(0)=\sum_\ell\lambda^{\rm lift}_\ell\|\hat f_\ell\|^2=\J_{\rm unif}(q_\varepsilon)/\varepsilon^2$ and $\mathcal C'(0)=-\sum_\ell G_\ell\|\hat f_\ell\|^2=:-\mathsf G<0$ by part (i). By Step 2 with $q=\sigma$, $A_\vartheta(t)=\frac{d-1}2\vartheta\,t^2g(t)+O(\vartheta^2)$ with a polynomially bounded remainder, so by \ref{F3}
\[
\mathcal A(\vartheta)=a_2\vartheta^2+O(\vartheta^3),\qquad a_2=\frac{(d-1)^2}{4}\E\,t^4e^{-t^2}=\frac{(d-1)^2}{12\sqrt3}.
\]
(This matches \cref{lem:align}: to leading order the channel is the extremal family with $\Delta=(d-1)\vartheta$.) Choose $\vartheta_0\in(0,1]$ with $\mathcal A''\ge a_2$ on $[0,\vartheta_0]$, and let $M=\max_{[0,1]}|\mathcal C''|$. Then $\mathcal A'(\vartheta)\ge a_2\vartheta$ and $\mathcal C'(\vartheta)\ge-\mathsf G-M\vartheta$ on $[0,\vartheta_0]$, so $\partial^2_\vartheta\J\ge a_2-\varepsilon^2M\ge a_2/2$ on $[0,\vartheta_0]$ once $\varepsilon^2M\le a_2/2$, and $\J$ is strictly convex there. Moreover $\partial_\vartheta\J\ge(a_2-\varepsilon^2M)\vartheta-\varepsilon^2\mathsf G>0$ for $\vartheta>2\varepsilon^2\mathsf G/a_2$; take $\varepsilon$ also so small that $2\varepsilon^2\mathsf G/a_2<\vartheta_0$. Since $\partial_\vartheta\J(0,\varepsilon)=-\varepsilon^2\mathsf G<0$, the unique minimizer $\vartheta^\star_\varepsilon$ over $[0,\vartheta_0]$ is interior and satisfies $\vartheta^\star_\varepsilon\le2\varepsilon^2\mathsf G/a_2$. At an interior minimizer, $0=\mathcal A'(\vartheta^\star)+\varepsilon^2\mathcal C'(\vartheta^\star)=2a_2\vartheta^\star-\varepsilon^2\mathsf G+O(\vartheta^{\star2}+\varepsilon^2\vartheta^\star)$, so $\vartheta^\star_\varepsilon=\frac{\varepsilon^2\mathsf G}{2a_2}(1+O(\varepsilon^2))$. With $\Delta(\vartheta)=(d-1)\vartheta(1+O(\vartheta))$,
\[
\frac{\Delta(\vartheta^\star_\varepsilon)}{\J_{\rm unif}(q_\varepsilon)}\to\frac{(d-1)\mathsf G}{2a_2\,\mathcal C(0)}=\frac{6\sqrt3\,\mathsf G}{(d-1)\,\mathcal C(0)}=\rho(f)\ >\ 6\sqrt3,
\]
using $\mathsf G>(d-1)\,\mathcal C(0)$, which follows from part (i) since $\ell(\ell+d-2)\ge d-1$. For $f\in\mathcal H_\ell$, part (i) gives $6\sqrt3\big[\frac{\ell(\ell+d-2)}{d-1}+\frac{\ell-1}3\big]<\rho(f)<6\sqrt3\big[\frac{\ell(\ell+d-2)}{d-1}+\ell\big]$.

\paragraph{Step 4: part (iii).} \emph{InfoNCE.} The paired law $P_\vartheta^{q_\varepsilon}$ has equal marginals $r_\vartheta\sigma$ with $r_\vartheta=1+\varepsilon\mathsf P_\vartheta f$, so by \cref{lem:decomp} $\Lb-\mathcal L^\star_\beta=2\beta\Delta(\vartheta)+\Ub(r_\vartheta\sigma)$. Let $g_\vartheta=\kappa+\varepsilon\mathcal K\mathsf P_\vartheta f$ and $\dot r=\varepsilon\nabla^2_{\mathbb S}\mathsf P_\vartheta f$, so $\|\dot r\|_\infty\le\varepsilon\|\nabla^2_{\mathbb S}f\|_\infty$ and $\int\dot r\,\dd\sigma=0$. Differentiating $\Ub=\int\log g_\vartheta\,r_\vartheta\,\dd\sigma-\log\kappa$,
\[
\frac{\dd}{\dd\vartheta}\Ub=\int\log\frac{g_\vartheta}{\kappa}\,\dot r\,\dd\sigma+\int(\mathcal K\dot r)\Big(\frac{r_\vartheta}{g_\vartheta}-\frac1\kappa\Big)\dd\sigma,
\]
where we subtracted $\int\dot r\,\dd\sigma\cdot\log\kappa=0$ and $\kappa^{-1}\int\mathcal K\dot r\,\dd\sigma=\int\dot r\,\dd\sigma=0$. For $\varepsilon\|f\|_\infty\le\frac14$, $|\log(g_\vartheta/\kappa)|\le2\varepsilon\|f\|_\infty$, $|r_\vartheta/g_\vartheta-1/\kappa|\le4\varepsilon\|f\|_\infty/\kappa$ and $|\mathcal K\dot r|\le\kappa\|\dot r\|_\infty$. Hence $|\frac{\dd}{\dd\vartheta}\Ub|\le6\varepsilon^2\|f\|_\infty\|\nabla^2_{\mathbb S}f\|_\infty$, while $\frac{\dd}{\dd\vartheta}2\beta\Delta=2\beta(d-1)e^{-2(d-1)\vartheta}\ge2\beta(d-1)e^{-2(d-1)}$ on $[0,1]$. For small $\varepsilon$ the objective is strictly increasing on $[0,\vartheta_0]$, and $\vartheta=0$ is its unique minimizer. \hfill$\square$

\begin{proof}[Proof of \cref{cor:threshold}]
With $\alpha=a\varepsilon^2$, $\J_\alpha(\vartheta)=\mathcal A(\vartheta)+\varepsilon^2\mathcal C(\vartheta)+a\varepsilon^2\Delta(\vartheta)$, and $|\Delta''|\le2(d-1)^2$, so $\partial^2_\vartheta\J_\alpha\ge a_2-\varepsilon^2(M+2a(d-1)^2)>0$ on $[0,\vartheta_0]$ for small $\varepsilon$: $\J_\alpha$ is strictly convex there and has a unique minimizer. Its derivative at $0$ is $\varepsilon^2(a(d-1)-\mathsf G)$. If $a(d-1)\ge\mathsf G$, the minimizer is $\vartheta=0$. Otherwise it is interior, and as in Step~3, $0=2a_2\vartheta^\star-\varepsilon^2(\mathsf G-a(d-1))+O(\vartheta^{\star2}+\varepsilon^2\vartheta^\star)$ gives $\vartheta^\star=\varepsilon^2(\mathsf G-a(d-1))/(2a_2)\,(1+O(\varepsilon^2))$. Hence $\Delta(\vartheta^\star)=\varepsilon^2\big[(d-1)\mathsf G-a(d-1)^2\big]_+/(2a_2)+o(\varepsilon^2)$. Since $(d-1)\varepsilon^2\mathsf G/(2a_2)=\rho(f)\J_{\rm unif}(q_\varepsilon)$ and $(d-1)^2/(2a_2)=6\sqrt3$, this is the claim. More generally, for every $\alpha\ge\varepsilon^2\mathsf G/(d-1)$ and $\vartheta\in(0,\vartheta_0]$, $\partial_\vartheta\J_\alpha\ge(a_2-\varepsilon^2M)\vartheta-\varepsilon^2\mathsf G(1-e^{-2(d-1)\vartheta})\ge\vartheta\,(a_2-\varepsilon^2M-2(d-1)\varepsilon^2\mathsf G)>0$ for small $\varepsilon$, so $\vartheta=0$ is the unique minimizer; for $\alpha<\varepsilon^2\mathsf G/(d-1)$ the derivative at $0$ is negative and the minimizer is positive. In particular every fixed $\alpha>0$ gives $\vartheta=0$ for small $\varepsilon$.
\end{proof}

\paragraph{Numerical constants.} The values $\rho_\ell$ quoted after \eqref{eq:balance} are $\rho(f)$ for $f\in\mathcal H_\ell$, i.e., $\rho_\ell=6\sqrt3\,G_\ell/((d-1)\lambda^{\rm lift}_\ell)$, computed by quadrature from \cref{lem:lift-eig}.

\subsection{Proof of Proposition~\ref{prop:calibrated} and Corollary~\ref{cor:linear} (linear certificates)}\label{app:calibrated-proof}

If $\J=0$ then $\Delta=0$ by \cref{lem:align}, so $\J_\alpha=0$; the converse is trivial. By \cref{lem:decomp}, $\log(1+x)\le x$, \cref{prop:kernel} and \cref{thm:structure},
\[
\Lb-\mathcal L^\star_\beta\le2\beta\Delta+\frac{C^\star_{d,\beta}}{\kappa_d(\beta)}\Dlift^2(q,\sigma)\le2\beta\Delta+\frac{C^\star_{d,\beta}}{\kappa_d(\beta)}\big(\sqrt\J+\tfrac32\Delta\big)^2 .
\]
Using $(x+y)^2\le2x^2+2y^2$ and $\Delta^2\le\Delta$, the right-hand side is at most $(2\beta+\frac92c_\beta)\Delta+2c_\beta\J\le c_{\beta,\alpha}(\alpha\Delta+\J)$ with $c_\beta=C^\star_{d,\beta}/\kappa_d(\beta)$. \hfill$\square$

\begin{proof}[Proof of \cref{cor:linear}]
By \cref{lem:decomp}, $\log(1+x)\le x$ and \eqref{eq:mmd-spectral}, $\Lb-\mathcal L^\star_\beta\le2\beta\Delta+\Db^2/\kappa\le2\beta\Delta+(S_K/\kappa)\,\mathrm{MMD}_K^2(q,\sigma)$, which is at most $\max\{2\beta/\alpha,S_K/\kappa\}(\alpha\Delta+\mathrm{MMD}_K^2)$. If $S_K=\infty$, restrict to aligned laws, where $\Delta=0$ and the bound would read $\Ub\le c\,\mathrm{MMD}_K^2$; this contradicts the lower bound of \cref{thm:spectral}(ii), which gives $c_K\ge\sup_\ell\lambda^\beta_\ell/(2\kappa\mu_\ell)=\infty$.
\end{proof}

\subsection{General channels}\label{app:gain}

This section proves \cref{prop:condA} and the gain bounds stated after it: \cref{prop:centered} bounds the gain of every centered channel, and \cref{prop:hetero} shows that heteroscedastic heat channels attain the order $\sqrt{\J_{\rm unif}}$. We use the notation and facts of \cref{app:prelim,app:spectral,app:slaving}.

\paragraph{Notation.} For a probability measure $q$ on $\Sd$ and $\ell\ge1$ let
\[
u_\ell(x)=\int N_\ell P_\ell(x^\top y)\,\dd q(y)=\sum_{m\le N_\ell}Y_{\ell m}(x)\int Y_{\ell m}\,\dd q ,
\]
the degree-$\ell$ component of $q-\sigma$ (for $q=(1+u)\sigma$ this is the projection of $u$ onto $\mathcal H_\ell$), so that $\|u_\ell\|_{L^2(\sigma)}=\|\hat q_\ell\|$ and $\J_{\rm unif}(q)=\Dlift^2(q,\sigma)=\sum_{\ell\ge1}\lambda^{\rm lift}_\ell\|\hat q_\ell\|^2$ by \eqref{eq:mmd-spectral} and \cref{lem:lift-eig}. By the Cauchy--Schwarz inequality and the addition theorem,
\begin{equation}\label{eq:ul-sup}
|u_\ell(x)|\le\sqrt{N_\ell}\,\|\hat q_\ell\|,\qquad \|\hat q_\ell\|^2=\iint N_\ell P_\ell(x^\top y)\,\dd q(x)\,\dd q(y)\le N_\ell .
\end{equation}
$G_\ell>0$ are the constants of \cref{thm:slaving}(i) and \eqref{eq:G-bracket}. A \emph{family} is a map $\vartheta\mapsto P_\vartheta$, $\vartheta\in[0,\vartheta_0]$, into paired laws; for an objective $F$ we write $f(\vartheta)=F(P_\vartheta)$ and $\Delta(\vartheta)=\Delta(P_\vartheta)$.

\begin{lemma}[Summability]\label{lem:summable}
$G_\ell<\ell(\ell+2d-3)\,\lambda^{\rm lift}_\ell$ for $\ell\ge1$, and $\sum_{\ell\ge1}N_\ell\,\ell^k\lambda^{\rm lift}_\ell<\infty$ for every $k\ge0$. Consequently
\[
K_d:=\Big(\sum_{\ell\ge1}\frac{N_\ell G_\ell^2}{\lambda^{\rm lift}_\ell}\Big)^{1/2}<\infty,
\qquad
\Gamma_q(x):=\sum_{\ell\ge1}G_\ell\,u_\ell(x)
\]
converges absolutely and uniformly on $\Sd$, and $\Gamma_q$ is continuous with $\int\Gamma_q\,\dd\sigma=0$.
\end{lemma}

\begin{proof}
The first claim is the upper bound in \eqref{eq:G-bracket}. By \cref{lem:eig-bounds}(ii), $\lambda^{\rm lift}_\ell\le c_\ell^2(2\ell-1)!!$, and $\frac1\ell\log\big(c_\ell^2(2\ell-1)!!\big)\to-\log2$ (proof of \cref{prop:kernel}), so $\lambda^{\rm lift}_\ell\le C_d\,(3/5)^\ell$ for all $\ell$; $N_\ell\le c_d\ell^{d-2}$. This gives the series bound, and $K_d^2\le\sum_\ell N_\ell\,\ell^2(\ell+2d-3)^2\lambda^{\rm lift}_\ell<\infty$. By \eqref{eq:ul-sup}, $\sum_\ell G_\ell\|u_\ell\|_\infty\le\sum_\ell G_\ell N_\ell<\infty$. Each $u_\ell$ is continuous with $\sigma$-mean zero.
\end{proof}

\begin{proof}[Proof of \cref{prop:condA}]
Write $f(\vartheta)=F(P_\vartheta)$ and $f_\alpha=f+\alpha\Delta$. By definition of $g$, the right derivative is $f_\alpha'(0)=f'(0)+\alpha\Delta'(0)=\Delta'(0)(\alpha-g)$, so $f_\alpha(\vartheta)-f_\alpha(0)=\Delta'(0)(\alpha-g)\vartheta+o(\vartheta)$, which gives (i). For (ii), $\lambda S+(1-\lambda)I=\lambda\,(S+\frac{1-\lambda}\lambda I)$; dividing by $\lambda>0$ does not change minimizers, and (i) applies with $I$ as the misalignment coordinate. Multiplying $S$ by $B$ multiplies its derivative, and hence $g_S$, by $B$.
\end{proof}

\noindent For the isotropic heat channel, $g=\rho(f)\J_{\rm unif}(q_\varepsilon)/(6\sqrt3)$, and \cref{cor:threshold} gives the global version of (i) on $[0,\vartheta_0]$.

\subsubsection*{Heteroscedastic heat channels}

Let $w:\Sd\to[0,\infty)$ be measurable with $0<\E_qw<\infty$. The \emph{heteroscedastic heat channel} $P^{q,w}_\vartheta$ draws $W\sim q$ and then $U,V$ independently from $p_{\vartheta w(W)}(W,\cdot)$; $w\equiv1$ is the channel of \cref{thm:slaving}. It is a centered channel in the sense of \cref{prop:centered} below, with $c(W)=e^{-(d-1)\vartheta w(W)}$.

\begin{proposition}[Gain of heteroscedastic heat channels]\label{prop:hetero}
Let $q\ne\sigma$, and let $w$ be as above.
\begin{enumerate}[label=(\roman*),leftmargin=*,itemsep=1pt,topsep=2pt]
\item $\Delta(\vartheta)=\frac12\big(1-\E_qe^{-2(d-1)\vartheta w(W)}\big)$ and $\vartheta\mapsto\J(P^{q,w}_\vartheta)$ are continuous on $[0,\infty)$ and right-differentiable at $0$, $\Delta$ is strictly increasing with $\Delta'(0)=(d-1)\E_qw$, and the gain of the family is
\[
g(w)=\frac{\E_q[w\,\Gamma_q]}{(d-1)\,\E_qw}.
\]
If $w$ is bounded, both functions are $C^\infty$ on $[0,\infty)$.
\item For $w\equiv1$, $g=\frac{1}{d-1}\sum_{\ell\ge1}G_\ell\|\hat q_\ell\|^2=\frac{\rho(q)}{6\sqrt3}\J_{\rm unif}(q)$, with $\rho(q):=\frac{6\sqrt3\sum_\ell G_\ell\|\hat q_\ell\|^2}{(d-1)\sum_\ell\lambda^{\rm lift}_\ell\|\hat q_\ell\|^2}>6\sqrt3$.
\item $\sup_wg(w)=\frac{1}{d-1}\max_{\mathrm{supp}\,q}\Gamma_q$, the supremum being over bounded $w$ or over all admissible $w$. It is attained if and only if $q(\{x\in\mathrm{supp}\,q:\Gamma_q(x)=\max_{\mathrm{supp}\,q}\Gamma_q\})>0$.
\item $\|\Gamma_q\|_\infty\le K_d\sqrt{\J_{\rm unif}(q)}$; hence $\sup_wg(w)\le\frac{K_d}{d-1}\sqrt{\J_{\rm unif}(q)}$. The constant is sharp: $\sup\{g(w)/\sqrt{\J_{\rm unif}(q)}\}=K_d/(d-1)$ over all $q\neq\sigma$ and $w$, and codes with $\sup_wg(w)\ge(K_d/(d-1)-\eta)\sqrt{\J_{\rm unif}(q)}$ exist for every $\eta>0$ and with $\J_{\rm unif}(q)$ arbitrarily small.
\end{enumerate}
\end{proposition}

\begin{proof}
Write $\mathsf Q_s=\mathsf P_s\otimes\mathsf P_s$ for the product heat semigroup on $\Sd\times\Sd$, with generator $L_2=\nabla^2_{\mathbb S,u}+\nabla^2_{\mathbb S,v}$, and, as in Step~1 of \cref{app:slaving}, $G_{a,t}(u,v)=F_{a,t}(\frac{u+v}2)$ with $F_{a,t}(c)=\E_Re^{itRa^\top c}$. Throughout, $p_k$ denotes a polynomial, depending only on $d$ and $k$, with $\|L_2^kG_{a,t}\|_\infty\le p_k(|t|)$ (Step~1 of \cref{app:slaving}).

\emph{Semigroup bounds.} For smooth $G$, $s\mapsto\mathsf Q_sG$ is $C^\infty$ into $C(\Sd\times\Sd)$ with $\partial_s^k\mathsf Q_sG=\mathsf Q_sL_2^kG$, and $\mathsf Q_s$ is a sup-norm contraction. Hence, for all $s\ge0$ and $x\in\Sd$,
\begin{equation}\label{eq:sg-taylor}
|\mathsf Q_sG(x,x)-G(x,x)|\le s\|L_2G\|_\infty,\qquad |\mathsf Q_sG(x,x)-G(x,x)-s\,L_2G(x,x)|\le\tfrac{s^2}2\|L_2^2G\|_\infty,
\end{equation}
and $L_2G_{a,t}(x,x)=\mathcal GF_{a,t}(x)$ by Step~1 of \cref{app:slaving}.

\emph{(i), misalignment.} The coordinates are degree-one harmonics, so $\E[U\mid W]=e^{-(d-1)\vartheta w(W)}W$, and conditional independence gives $\E[U^\top V\mid W]=e^{-2(d-1)\vartheta w(W)}$, which is the formula for $\Delta$. It is continuous and strictly increasing because $q(w>0)>0$. Since $0\le(1-e^{-2(d-1)\vartheta w})/\vartheta\le2(d-1)w\in L^1(q)$, dominated convergence gives $\Delta'(0)=(d-1)\E_qw$.

\emph{(i), first variation of $\J$.} Conditionally on $W=x$, $\E[F_{a,t}(C)\mid W=x]=\mathsf Q_{\vartheta w(x)}G_{a,t}(x,x)$, so $\varphi_\vartheta(ta):=\varphi_Z(ta)=\int\mathsf Q_{\vartheta w(x)}G_{a,t}(x,x)\,\dd q(x)$. Put $D=\varphi_0-g$ and $\psi(a,t)=\int w\,\mathcal GF_{a,t}\,\dd q$. By \eqref{eq:sg-taylor},
\[
\frac1\vartheta\big|\varphi_\vartheta-\varphi_0-\vartheta\psi\big|\le\int\min\Big\{\tfrac{\vartheta w^2}2p_2(|t|),\,2w\,p_1(|t|)\Big\}\dd q\ \xrightarrow[\vartheta\downarrow0]{}\ 0
\]
by dominated convergence, and $|\varphi_\vartheta-\varphi_{\vartheta'}|\le|\vartheta-\vartheta'|\,\E_qw\;p_1(|t|)$. The last bound shows that $\J(\vartheta)=\E_{a,t}|\varphi_\vartheta-g|^2$ is continuous. Writing $\J(\vartheta)-\J(0)=2\,\mathrm{Re}\,\E_{a,t}[\overline D(\varphi_\vartheta-\varphi_0)]+\E_{a,t}|\varphi_\vartheta-\varphi_0|^2$, with $|D|\le2$ and $\E_tp_1(|t|)^2<\infty$, dominated convergence gives
\begin{equation}\label{eq:Jprime-het}
\partial_\vartheta\J\big|_{\vartheta=0^+}=2\,\mathrm{Re}\,\E_{a,t}\big[\overline{D(a,t)}\,\psi(a,t)\big].
\end{equation}
If $w$ is bounded, $\partial_\vartheta^k\varphi_\vartheta=\int w^k\,\mathsf Q_{\vartheta w}L_2^kG_{a,t}(x,x)\,\dd q$ is bounded by $\|w\|_\infty^kp_k(|t|)$, and $\J$ and $\Delta$ are $C^\infty$ by dominated convergence.

\emph{(i), spectral evaluation.} For finite signed measures $\mu,\nu$ on $\Sd$ and fixed $t$ let $\Phi_\nu(a)=\int F_{a,t}\,\dd\nu$ and $\Psi_\mu(a)=\int\mathcal GF_{a,t}\,\dd\mu$. For $Y\in\mathcal H_\ell$ and $x\in\Sd$, \eqref{eq:fh-plane} (applied in the variable $a$) gives $\int F_{a,t}(x)Y(a)\,\dd\sigma(a)=i^\ell h_\ell(t)Y(x)$. Since $\mathcal GF_{a,t}(x)=\frac12\nabla^2_{\mathbb S}F_{a,t}(x)-\frac{d-1}2t\,\partial_tF_{a,t}(x)$ on $\Sd$ and all derivatives are dominated, we may differentiate under the integral:
\[
\int\mathcal GF_{a,t}(x)\,Y(a)\,\dd\sigma(a)=i^\ell k_\ell(t)Y(x),\qquad k_\ell(t):=-\tfrac12\ell(\ell+d-2)h_\ell(t)-\tfrac{d-1}2\,t\,h_\ell'(t).
\]
By Fubini, the degree-$\ell$ components of the bounded functions $\Phi_\nu,\Psi_\mu\in L^2(\sigma)$ are $i^\ell h_\ell(t)\nu_\ell$ and $i^\ell k_\ell(t)\mu_\ell$, where $\mu_\ell=\sum_m(\int Y_{\ell m}\dd\mu)Y_{\ell m}$. Parseval's identity ($h_\ell,k_\ell$ are real) gives $\E_a[\overline{\Phi_\nu}\Psi_\mu]=\sum_{\ell\ge0}h_\ell(t)k_\ell(t)\langle\nu_\ell,\mu_\ell\rangle$, and, by Cauchy--Schwarz in $\ell$ and Parseval again, $\sum_\ell|h_\ell k_\ell\langle\nu_\ell,\mu_\ell\rangle|\le\|\Phi_\nu\|_{L^2(\sigma)}\|\Psi_\mu\|_{L^2(\sigma)}\le|\nu|(\Sd)\,|\mu|(\Sd)\,p_1(|t|)$, which is integrable in $t$. Hence $\E_t$ and $\sum_\ell$ commute. Gaussian integration by parts, as in Step~2 of \cref{app:slaving}, gives $\E_t[h_\ell k_\ell]=-\frac12\big(\ell(\ell+d-2)\lambda^{\rm lift}_\ell+\frac{d-1}2\E_t[(t^2-1)h_\ell^2]\big)=-\frac12G_\ell$ for $\ell\ge1$. Take $\nu=q-\sigma$ (so $\nu_0=0$, $\nu_\ell=u_\ell$, $\Phi_\nu=D$) and $\mu=w\,q$ (so $\Psi_\mu=\psi$). By \eqref{eq:Jprime-het},
\[
\partial_\vartheta\J\big|_{0^+}=-\sum_{\ell\ge1}G_\ell\langle u_\ell,\mu_\ell\rangle=-\sum_{\ell\ge1}G_\ell\int u_\ell\,w\,\dd q=-\E_q[w\,\Gamma_q],
\]
where $\langle u_\ell,\mu_\ell\rangle=\int u_\ell\,\dd\mu$ because $u_\ell\in\mathcal H_\ell$, and the last sum converges absolutely by \cref{lem:summable}. Dividing by $-\Delta'(0)$ gives $g(w)$.

\emph{(ii)} For $w\equiv1$, $\E_q\Gamma_q=\sum_\ell G_\ell\int u_\ell\,\dd q=\sum_\ell G_\ell\|\hat q_\ell\|^2$, since $\int u_\ell\,\dd q=\int u_\ell\,\dd(q-\sigma)=\|u_\ell\|^2$. The expression with $\rho(q)$ is a rewriting, and $\rho(q)>6\sqrt3$ by $G_\ell>\ell(\ell+d-2)\lambda^{\rm lift}_\ell\ge(d-1)\lambda^{\rm lift}_\ell$.

\emph{(iii)} $g(w)$ is the mean of $\Gamma_q/(d-1)$ under the probability measure $w\,\dd q/\E_qw$, hence at most $M/(d-1)$ with $M:=\operatorname*{ess\,sup}_q\Gamma_q=\max_{\mathrm{supp}\,q}\Gamma_q$ (continuity of $\Gamma_q$). For $\eta>0$, $w=\mathbf 1\{\Gamma_q>M-\eta\}$ is bounded with $\E_qw>0$ and $g(w)\ge(M-\eta)/(d-1)$. If $q(\Gamma_q=M)>0$, $w=\mathbf 1\{\Gamma_q=M\}$ attains $M/(d-1)$; conversely, $g(w)=M/(d-1)$ forces $\E_q[w(M-\Gamma_q)]=0$ with a $q$-a.e.\ nonnegative integrand, so $w=0$ $q$-a.e.\ on $\{\Gamma_q<M\}$, and $\E_qw>0$ forces $q(\Gamma_q=M)>0$.

\emph{(iv)} By \eqref{eq:ul-sup} and Cauchy--Schwarz,
\[
|\Gamma_q(x)|\le\sum_\ell G_\ell\sqrt{N_\ell}\,\|\hat q_\ell\|=\sum_\ell\frac{G_\ell\sqrt{N_\ell}}{\sqrt{\lambda^{\rm lift}_\ell}}\cdot\sqrt{\lambda^{\rm lift}_\ell}\|\hat q_\ell\|\le K_d\sqrt{\J_{\rm unif}(q)} .
\]
With (iii) this gives the bound. Sharpness: fix $x_0\in\Sd$ and $L\ge1$, let $f_L(x)=\sum_{\ell=1}^L(G_\ell/\lambda^{\rm lift}_\ell)N_\ell P_\ell(x_0^\top x)$ and $q=(1+\varepsilon f_L)\sigma$ with $0<\varepsilon\|f_L\|_\infty<1$, so that $\mathrm{supp}\,q=\Sd$. By the addition theorem, $N_\ell P_\ell(x_0^\top\cdot)\in\mathcal H_\ell$ has squared norm $N_\ell$, so $u_\ell=\varepsilon(G_\ell/\lambda^{\rm lift}_\ell)N_\ell P_\ell(x_0^\top\cdot)$ for $\ell\le L$ and $0$ otherwise. Hence $\J_{\rm unif}(q)=\varepsilon^2K_L^2$ and $\Gamma_q(x_0)=\varepsilon K_L^2$, with $K_L^2=\sum_{\ell\le L}N_\ell G_\ell^2/\lambda^{\rm lift}_\ell$. By (iii), $\sup_wg(w)\ge\frac{K_L}{d-1}\sqrt{\J_{\rm unif}(q)}$, and $K_L\uparrow K_d$ while $\varepsilon$ is free.
\end{proof}

\subsubsection*{A nuisance certificate for centered channels}

Let $\mathsf k_d:=2+\frac92\sqrt{d/(d+2)}<6.5$.

\begin{proposition}[Nuisance certificate for centered channels]\label{prop:centered}\label{thm:general}
Let $\Lambda$ be a random element, $W=W(\Lambda)\in\Sd$ with law $q$, and let $U,V\in\Sd$ be conditionally independent given $\Lambda$ with
\[
\E[U\mid\Lambda]=\E[V\mid\Lambda]=c(\Lambda)\,W,\qquad c(\Lambda)\in[0,1].
\]
Then, with $P$ the law of $(U,V)$ and $T_q=RW$,
\[
\big\|\varphi_Z(t\,\cdot)-\varphi_{T_q}(t\,\cdot)\big\|\le\mathsf k_d\,\Delta(P),\qquad\text{hence}\qquad\big|\sqrt{\J(P)}-\sqrt{\J_{\rm unif}(q)}\big|\le\mathsf k_d\,\Delta(P)\le6.5\,\Delta(P),
\]
and $\J(P)+\alpha\Delta(P)\ge\J_{\rm unif}(q)+\big(\alpha-2\mathsf k_d\sqrt{\J_{\rm unif}(q)}\big)\Delta(P)$ for every $\alpha\ge0$. In particular, if $\alpha>2\mathsf k_d\sqrt{\J_{\rm unif}(q)}$ (for instance $\alpha>13\sqrt{\J_{\rm unif}(q)}$), then $\J_\alpha(P)>\J_\alpha(P_q)=\J_{\rm unif}(q)$ for every such $P$ with $\Delta(P)>0$, where $P_q$ is the aligned law $U=V=W$.
\end{proposition}

\begin{proof}
Fix $(a,t)$ and put $x=tRa$; $R$ is independent of $(\Lambda,U,V)$. Since $|e^{iy}-1-iy|\le y^2/2$,
\[
\varphi_Z(ta)-\varphi_{T_q}(ta)=\E\big[e^{ix^\top W}\,i\,x^\top(C-W)\big]+\E\big[e^{ix^\top W}r\big]=:T_1(a,t)+T_2(a,t),\qquad|r|\le\tfrac12\big(x^\top(C-W)\big)^2 .
\]
\emph{Conditional moments.} Let $\Delta_\Lambda=\frac14\E[\|U-V\|^2\mid\Lambda]$, so $\E\Delta_\Lambda=\Delta$. By conditional independence, $\E[U^\top V\mid\Lambda]=c^2$, so $\Delta_\Lambda=\frac12(1-c^2)$ and, as $c\in[0,1]$, $1-c\le(1-c)(1+c)=2\Delta_\Lambda$. Moreover $\E[C\mid\Lambda]=cW$ and $\E[\|C\|^2\mid\Lambda]=1-\Delta_\Lambda$, so
\[
\E\big[\|C-W\|^2\mid\Lambda\big]=2(1-c)-\Delta_\Lambda\le3\Delta_\Lambda .
\]
\emph{First-order term.} Conditionally on $(\Lambda,R)$, $\E[x^\top(C-W)\mid\Lambda,R]=tR(c-1)\,a^\top W$. Hence $|T_1|\le2|t|\,\E R\;\E\big[|a^\top W|\Delta_\Lambda\big]$. By the Cauchy--Schwarz inequality with weights $\Delta_\Lambda$, $\big(\E[|a^\top W|\Delta_\Lambda]\big)^2\le\Delta\,\E[(a^\top W)^2\Delta_\Lambda]$, and $\E_a(a^\top W)^2=1/d$. Using $(\E R)^2\le\E R^2=d$ \ref{F5} and $\E t^2=1$, we get $\|T_1\|^2\le4\,\E t^2\,(\E R)^2\,\Delta^2/d\le4\Delta^2$.

\emph{Remainder.} With $\Sigma=\E[(C-W)(C-W)^\top]\succeq0$, $\mathrm{tr}\,\Sigma\le3\Delta$, independence of $R$ and \ref{F5} give $|T_2|\le\frac d2t^2\,a^\top\Sigma a$. Exactly as in \cref{lem:centroid}, \ref{F3}--\ref{F4} and $\mathrm{tr}(\Sigma^2)\le(\mathrm{tr}\,\Sigma)^2$ give
\[
\|T_2\|^2\le\frac{d^2}4\cdot3\cdot\frac{3(\mathrm{tr}\,\Sigma)^2}{d(d+2)}\le\frac{81\,d}{4(d+2)}\Delta^2 .
\]
By Minkowski's inequality, $\|\varphi_Z-\varphi_{T_q}\|\le(2+\frac92\sqrt{d/(d+2)})\Delta=\mathsf k_d\Delta$. Since $\sqrt\J=\|\varphi_Z-g\|$ and $\sqrt{\J_{\rm unif}(q)}=\Dlift(q,\sigma)=\|\varphi_{T_q}-g\|$, the reverse triangle inequality gives the second claim.

\emph{Consequence.} Let $s=\sqrt{\J_{\rm unif}(q)}$. If $s\ge\mathsf k_d\Delta$, then $\J\ge(s-\mathsf k_d\Delta)^2\ge s^2-2\mathsf k_ds\Delta$. If $s<\mathsf k_d\Delta$, then $s^2-2\mathsf k_ds\Delta<-\mathsf k_ds\Delta\le0\le\J$. In both cases $\J\ge\J_{\rm unif}(q)-2\mathsf k_d\sqrt{\J_{\rm unif}(q)}\,\Delta$; add $\alpha\Delta$. Finally, $\J(P_q)=\Dlift^2(q,\sigma)=\J_{\rm unif}(q)$ and $\Delta(P_q)=0$.
\end{proof}

\begin{corollary}[Gains of centered channels]\label{cor:centered-gain}
Let $(P_\vartheta)$ be a family as in \cref{prop:condA} with $F=\J$, such that each $P_\vartheta$ is a centered channel as in \cref{prop:centered} with the same code law $q$, and $P_0=P_q$. Then its gain satisfies $g\le2\mathsf k_d\sqrt{\J_{\rm unif}(q)}\le13\sqrt{\J_{\rm unif}(q)}$. Consequently, $\J_\alpha$ has the aligned code as a strict local minimizer along every such family once $\alpha>2\mathsf k_d\sqrt{\J_{\rm unif}(q)}$. By \cref{prop:hetero}(iv), the best constant $c^{\rm gain}_d$ in $g\le c^{\rm gain}_d\sqrt{\J_{\rm unif}(q)}$ over all centered families satisfies $K_d/(d-1)\le c^{\rm gain}_d\le2\mathsf k_d$, so the order $\sqrt{\J_{\rm unif}}$ cannot be improved.
\end{corollary}

\begin{proof}
By \cref{prop:centered}, $(f(\vartheta)-f(0))/\Delta(\vartheta)\ge-2\mathsf k_d\sqrt{\J_{\rm unif}(q)}$ for small $\vartheta>0$ (where $\Delta(\vartheta)>0$ because $\Delta'(0)>0$), and the left side tends to $-g$. Heteroscedastic heat channels are centered, with $c=e^{-(d-1)\vartheta w(W)}\in(0,1]$ and $\Lambda=W$.
\end{proof}

\begin{remark}[Scope of the centering hypothesis]\label{rem:mean-code}
Whenever the two views are conditionally independent given a latent $\Lambda$ with the same conditional mean (e.g.\ i.i.d.\ per-view style given content), the hypothesis of \cref{prop:centered} holds with the \emph{mean code} $W=\E[U\mid\Lambda]/\|\E[U\mid\Lambda]\|$ and $c=\|\E[U\mid\Lambda]\|$ ($W$ arbitrary where $c=0$). The hypothesis cannot be dropped: an aligned displacement of the code, $U=V=W'$ with $W'\ne W$, has $\Delta=0$ and $\J=\J_{\rm unif}(\mathrm{law}(W'))$, which in general differs from $\J_{\rm unif}(q)$.
\end{remark}

\subsection{Finite batches}\label{app:finite-proofs}

We use the setting of \cref{sec:finite}, with $B\ge2$, $M\ge1$, and frequencies independent of the pairs and radii.

In the notation of \cref{sec:finite}, \eqref{eq:J-U} reads
\begin{equation}\label{eq:JU}
\hat\J_U=\frac1M\sum_{m=1}^M\Big[\frac{B|\hat\varphi_m|^2-1}{B-1}-2g_m\mathrm{Re}\,\hat\varphi_m+g_m^2\Big].
\end{equation}

\noindent Unbiasedness averages over both the i.i.d.\ examples and the fresh frequencies; a fixed quadrature in $(a,t)$ defines a different discrepancy.

\begin{proof}[Proof of Proposition~\ref{prop:bias}]
Condition on a frequency $\omega$. The $Z_j$ are i.i.d.\ (pairs are i.i.d.\ and each has its own radius), so $\E|\hat\varphi|^2=B^{-2}\sum_{j,k}\E e^{i\omega^\top(Z_j-Z_k)}=\frac1B+(1-\frac1B)|\varphi_Z(\omega)|^2$ and $\E\hat\varphi=\varphi_Z(\omega)$. Since $g$ is real,
$\E|\hat\varphi-g|^2=\E|\hat\varphi|^2-2g\,\mathrm{Re}\,\varphi_Z+g^2=|\varphi_Z-g|^2+\frac1B(1-|\varphi_Z|^2)$. Averaging over i.i.d.\ frequencies proves the identity. On the family of \cref{sec:infonce-comparison}, $\E_{a,t}|\varphi_Z|^2=\E e^{-st^2}=(1+2s)^{-1/2}$, so by \eqref{eq:Js}
\[
F_B(s):=\E\hat\J_V=\big(1-\tfrac1B\big)(1+2s)^{-1/2}-2(2+s)^{-1/2}+3^{-1/2}+\tfrac1B .
\]
Then $F_B'(s)=(2+s)^{-3/2}-(1-\frac1B)(1+2s)^{-3/2}$, which is negative if and only if $r(s):=\frac{1+2s}{2+s}<a_B:=(1-\frac1B)^{2/3}$. The map $r$ increases from $\frac12$ to $1$ on $[0,1]$, and $a_B\in(\frac12,1)$ for every $B\ge2$ (since $a_2=2^{-2/3}$). So $F_B$ decreases and then increases, with unique minimizer $s_B=r^{-1}(a_B)=(2a_B-1)/(2-a_B)\in(0,1)$. Moreover $1-s_B=3(1-a_B)/(2-a_B)$ and $1-a_B=\frac{2}{3B}+O(B^{-2})$, so $1-s_B=\frac2B+O(B^{-2})$ and $2s_B-1=1-\frac4B+O(B^{-2})$. Every aligned uniform law has $Z\sim\mathcal N(0,I_d)$ and thus the value $F_B(1)>F_B(s_B)$, so it is not a minimizer over paired laws.
\end{proof}

\begin{proof}[Proof of Proposition~\ref{prop:unbiased}]
\emph{Unbiasedness.} Conditionally on $\omega_m$, $\E[(B|\hat\varphi_m|^2-1)/(B-1)]=|\varphi_Z(\omega_m)|^2$ by the computation above, so the bracket in \eqref{eq:JU} has conditional mean $|\varphi_Z(\omega_m)-g_m|^2$. Averaging over $\omega_m=t_ma_m$ gives $\J$.

\emph{U-statistic form and cost.} With $\xi_j\in\mathbb C^M$, $(\xi_j)_m=M^{-1/2}(e^{i\omega_m^\top Z_j}-g_m)$, we have $\hat\J_U=\frac{1}{B(B-1)}\sum_{j\ne k}H(Z_j,Z_k)$ with the symmetric kernel $H(z,z')=\mathrm{Re}\langle\xi(z),\xi(z')\rangle$. Indeed, $B^2|\hat\varphi_m|^2=B+\sum_{j\ne k}\cos(\omega_m^\top(Z_j-Z_k))$ and $\sum_{j\ne k}(\cos\omega_m^\top Z_j+\cos\omega_m^\top Z_k)=2(B-1)\sum_j\cos\omega_m^\top Z_j$. Equivalently, $\hat\J_U=(\|\sum_j\xi_j\|^2-\sum_j\|\xi_j\|^2)/(B(B-1))$. Computing $\{\omega_m^\top Z_j\}$ is one $B\times d$ by $d\times M$ product, $O(BdM)$; the rest is $O(BM)$.
\end{proof}

\section{Experimental Details}\label{app:exp-details}

\paragraph{Exact families (\cref{fig:theory}a).} (i) Extremal family: \eqref{eq:Js} and $\Lb-\mathcal L^\star_\beta=2\beta(1-s)$. (ii) Noisy views: $U\sim\sigma$, $V=(U+\eta_{\rm n}\xi)/\|U+\eta_{\rm n}\xi\|$ with $\xi\sim\mathcal N(0,I_d)$. The pair law is rotation invariant, so $Z=R\|C\|\theta$ with $\theta\sim\sigma$ independent of $\|C\|$, and $\varphi_Z(ta)=\E e^{-t^2\|C\|^2/2}$. We evaluate $\J$ by 80-point Gauss--Hermite quadrature in $t$ and $4\times10^5$ Monte Carlo samples of $\|C\|^2=(1+U^\top V)/2$ with common random numbers across $t$; $\Ub=0$. (iii) Aligned vMF: $U=V\sim\mathrm{vMF}(e,\kappa_{\rm vMF})$. Then $\|\hat q_\ell\|^2=N_\ell\,r_\ell^2$ with $r_\ell=I_{\ell+\nu}(\kappa_{\rm vMF})/I_\nu(\kappa_{\rm vMF})$, $\J=\sum_{\ell\ge1}\lambda^{\rm lift}_\ell N_\ell r_\ell^2$ (truncated at $\ell=80$), and $\Ub$ is a one-dimensional integral; $\lambda^{\rm lift}_\ell$ is computed from \cref{lem:lift-eig} by adaptive quadrature. (iv) Mixture path: $\J(P_p)=p^2\J(\sigma\otimes\sigma)$ with $\J(\sigma\otimes\sigma)=0.02046$ at $d=16$.

\paragraph{Spectral threshold (\cref{fig:spectral}b).} For $d=8$, $\beta=5$, $\varepsilon=0.3$ and $q=(1+\varepsilon P_\ell(e^\top\cdot))\sigma$ with $P_\ell$ the Gegenbauer polynomial normalized by $P_\ell(1)=1$, $\Ub(q)=\int\log(1+\varepsilon \varrho_\ell P_\ell)(1+\varepsilon P_\ell)\,\dd\sigma$ with $\varrho_\ell=\lambda^\beta_\ell/\kappa=I_{\ell+D-1}(\beta)/I_{D-1}(\beta)$, computed with 3000-point Gauss--Jacobi quadrature for the law of $e^\top U$, and $\mathrm{MMD}_\gamma^2=\varepsilon^2\|P_\ell\|^2\,\Gamma(D)\gamma^{1-D}e^{-2\gamma}I_{\ell+D-1}(2\gamma)$ in closed form (Funk--Hecke).

\paragraph{Heat channel (\cref{fig:theory}c--d).} We use $q=(1+\varepsilon f_\ell)\sigma$ with $f_\ell=\sqrt{N_\ell}P_\ell(e^\top\cdot)$, so $\|f_\ell\|_{L^2(\sigma)}=1$, and restrict to $\varepsilon\|f_\ell\|_\infty<1$. We evaluate $\J$ from the pair law using the heat-kernel Gegenbauer expansion and Gaussian quadrature, and estimate minimizers by spline interpolation in heat time. Panels (c,d) numerically illustrate the local asymptotic law of \cref{thm:slaving}.

\paragraph{Estimator bias (\cref{fig:theory}b).} $d=16$, $M=64$ fresh frequencies per batch; $6\times10^4$ ($B=16$) and $3\times10^4$ ($B=64$) batches; the same batches are reused across $s$ (common random numbers).

\paragraph{Latent-variable study (\cref{sec:latent}).}
\emph{Data.} Content $c\sim\sigma$ on $\mathbb S^7$. (a) \emph{Content--style}: each view has independent style $s_i\sim\mathcal N(0,I_4)$, and $x_i=m(c,s_i)\in\R^{16}$. (b) \emph{vMF}: $x_1=m(c)$, $x_2=m(\tilde c)$ with $\tilde c\sim\mathrm{vMF}(c,\kappa_{\rm vMF}=10)$ \citep{zimmermann2021cl}. The mixing map $m$ acts on $(c,s_i)\in\R^{12}$ (vMF: on $c\in\R^8$) and is a fixed random three-layer leaky-ReLU (slope $0.2$) network with orthogonal square layers after a Gaussian input layer, hence injective almost surely. Fresh samples are drawn at every step, and the three training seeds share the mixing map. LeJEPA and VICReg hyperparameters are selected on validation samples; all reported metrics, probes and gains are computed on the frozen encoders with independently generated test samples ($3000$ pairs per model, shared across methods).
\emph{Encoder.} MLP $16\to256\to256\to256\to8$ with leaky-ReLU (slope $0.2$), shared across views; InfoNCE, SG, SG$_\alpha$ and moment matching act on the normalized outputs.
\emph{Optimization.} Adam (default moments, no weight decay or clipping), one-cycle schedule with peak learning rate $10^{-3}$ (5\% warm-up), $8000$ steps, $B=256$ unless stated, three seeds. InfoNCE$\to$SG$_\alpha$ trains InfoNCE for $4000$ steps and SG$_\alpha$ for the remaining $4000$ steps of the same schedule.
\emph{Losses.} InfoNCE: symmetric NT-Xent with $\tau=0.2$ ($\beta=5$). SG: $\hat\J_U$ or $\hat\J_V$ with $M=256$ fresh directions and frequencies per step; see \emph{Number of frequencies} below. SG$_\alpha$: $\hat\J_U+\alpha\cdot\frac14\|U-V\|^2$. LeJEPA follows the minimal reference implementation: the invariance term $I$ is the mean over views, batch and coordinates of the squared deviation from the view average, and SIGReg is applied to each view separately with $256$ random slices resampled at every step, $17$ knots on $[0,3]$ with trapezoidal weights and window $e^{-t^2/2}$, scaled by the batch size; the loss is $(1-\lambda)I+\lambda\,\mathrm{SIGReg}$ on unnormalized embeddings. VICReg follows the reference implementation on the $d=8$ outputs, without an expander (variance hinge with target $1$ and $\epsilon=10^{-4}$, squared off-diagonal covariances divided by $d$). For both we searched one seed over $\lambda\in\{0.005,0.01,0.02,0.05,0.1\}$ and $(\mathrm{sim},\mathrm{std},\mathrm{cov})\in\{(10,25,1),(25,25,1),(50,25,1)\}$ and selected by validation $R^2_c$: $\lambda=0.01$ and $(10,25,1)$, used in both latent models. All methods share the encoder, optimizer, schedule, batch size and number of steps; the other methods use fixed settings.
\emph{Metrics} (test samples, after normalizing). $\J$ by the Rao--Blackwellized U-statistic with $5000$ frequencies, the radius integrated by generalized Gauss--Laguerre quadrature. $\Ub$ and the InfoNCE excess at $\beta=5$ use all other test samples as negatives, averaged over the two views. $R^2_c$: linear regression of $c$ on $U$, fitted on one half of the test set and evaluated on the other. CKA: linear CKA \citep{kornblith2019similarity} of the view-1 embeddings to the InfoNCE encoders of the other seeds. Style: held-out $R^2$ of an MLP (one hidden layer of width $64$) predicting $s_1$ from $(U,c)$, fitted on one half and evaluated on the other; we call style detected when $R^2>0.02$.

\paragraph{Number of frequencies.} With $M\in\{16,64,128,256,1024\}$ at $B=256$ (three seeds each, same evaluation protocol as \cref{tab:latent}), the CKA of SG$_{0.2}$ to InfoNCE is $0.920,0.968,0.970,0.993,0.993$ on vMF and $0.818,0.827,0.834,0.828,0.825$ on content--style; on content--style, the style $R^2$ of pure SG is $0.29,0.17,0.15,0.11,0.10$. On vMF, the lower means at $M=64$ and $128$ come from one seed that converges to a different solution (CKA $0.92$); the same seed reaches $0.99$ at $M\ge256$, the other two seeds reach $0.99$ at every $M\ge64$, and three further seeds at $M=128$ reach $0.992$--$0.993$ (three further seeds at $M=256$: $0.994$). Thus $M=B/2$ matches $M=256$ in five of six vMF runs and on content--style.

\paragraph{Style probes: robustness.} On the seed-0 encoders of InfoNCE, SG, InfoNCE$\to$SG, SG$_{0.2}$, LeJEPA ($\lambda=0.01$) and VICReg (\cref{tab:latent-sd}), a larger probe (MLP $256{\times}256$, five-fold cross-validation on $10^4$ fresh pairs) gives the same classification, and refits of the reported probe on permuted style targets give $R^2\le-0.004$, below the detection threshold $0.02$. SG$_{0.2}$ is reported as the representative price because $0.2$ lies well above the critical price $\alpha\approx0.023$ of \cref{fig:slaving}a; the conclusions on the price rest on the full sweep.

\paragraph{Measuring the gain (\cref{fig:slaving}).} For a trained encoder, we rescale the style input, $x_i=m(c,\eta s_i)$ with $\eta\in\{0,0.03,0.06,0.1,0.15,0.2\}$, and evaluate the objective on $3000$ fresh pairs with common frequencies across $\eta$; we repeat this five times with independent pairs and frequencies and report the mean (bars in \cref{fig:slaving}: one standard deviation over repetitions). At $\eta=0$ the two views coincide. We fit $F-F(0)=-g\Delta+c_2\Delta^2$ by least squares on $\eta\le0.2$ (the comparisons of \cref{fig:slaving}a,b also hold for a refit on $\eta\le0.1$), with $F=\hat\J_U$ and $\Delta=\frac14\E\|U-V\|^2$ on the normalized embeddings; for LeJEPA, $F$ is the per-sample SIGReg statistic ($1024$ slices, $8$ repetitions with common random numbers) and $\Delta$ is replaced by $I$, which gives $g_S$. Since $s_i$ and $-s_i$ have the same law, $F$ and $\Delta$ are even in $\eta$, and $\vartheta=\eta^2$ plays the role of the parameter in \cref{prop:condA}, with $\Delta'(0)>0$ in $\vartheta$ whenever the embedding depends on the style at first order: $\eta=0$ gives the aligned law $P_0$ of the trained encoder, and the fitted slope estimates $g=-\dd F/\dd\Delta|_{\Delta=0}$ (for LeJEPA, $g_S=-\dd S/\dd I|_{I=0}$) along this empirical channel. LeJEPA is trained with the plug-in SIGReg statistic. Applying the identity behind \cref{prop:bias} to each slice, its batch-scaled expectation is $BS$ plus a term that does not grow with $B$. Measured along the same channel, this term contributes at most $0.33$ in absolute value to the gain of each of the $34$ encoders, and adding it to $Bg_S$ leaves every classification in \cref{fig:slaving}b unchanged.

\paragraph{SG$_\alpha$ sweep (\cref{fig:slaving}a).} We trained SG$_\alpha$ for $\alpha\in\{0,0.001,0.003,0.01,0.015,0.02,0.025,0.03,0.05,0.1,0.2\}$ with three seeds each (two at $\alpha=0.001$, five at $\alpha\in\{0.02,0.025,0.03\}$). The mean own-channel gain falls from $0.036$ at $\alpha=0$ to $0.006$ at $\alpha=0.2$, and interpolating it places the crossing $g(\alpha)=\alpha$ at $\alpha\approx0.023$. Style is detected in every seed for $\alpha\le0.003$, in some seeds for $\alpha\in\{0.01,0.015\}$, and in none for $\alpha\ge0.02$.

\paragraph{LeJEPA (\cref{fig:slaving}b).} At $B=256$ (22 encoders: one seed at $\lambda=0.005$, three otherwise; gain range over seeds):
\begin{center}\footnotesize\setlength{\tabcolsep}{3pt}\begin{tabular}{lcccccccc}\toprule $\lambda$ & 0.005 & 0.01 & 0.02 & 0.03 & 0.05 & 0.07 & 0.1 & 0.2\\\midrule
price $(1-\lambda)/\lambda$ & 199 & 99 & 49 & 32 & 19 & 13 & 9.0 & 4.0\\
gain $Bg_S$ & 6 & 5--9 & 13--15 & 19--23 & 22--28 & 28--40 & 37--61 & 73--127\\
style $R^2$ max & $-0.01$ & $-0.01$ & $-0.02$ & $-0.01$ & $0.14$ & $0.28$ & $0.61$ & $0.83$\\\bottomrule\end{tabular}\end{center}
\emph{Batch size.} At $B=64$ and $B=1024$:
\begin{center}\footnotesize\setlength{\tabcolsep}{3pt}\begin{tabular}{llcccc}\toprule
$B$ & $\lambda$ & price & gain $Bg_S$ & style $R^2$ & kept\\\midrule
$64$ (2 seeds) & $0.03,0.05,0.1$ & $32,19,9$ & $1.6$--$4.2$ & $\le-0.006$ & no\\
 & $0.2$ & $4$ & $9.6,9.8$ & $0.41,0.35$ & yes\\
$1024$ (seed 0) & $0.005,0.01$ & $199,99$ & $51,90$ & $\le-0.010$ & no\\
 & $0.02,0.03$ & $49,32$ & $127,221$ & $0.27,0.41$ & yes\\\bottomrule\end{tabular}\end{center}
\paragraph{Warm start (\cref{fig:slaving}c).} For $\alpha=0$, $0.001$, $0.003$, $0.01$, $0.03$, $0.2$, the style $R^2$ is $0.146$, $0.081$, $0.027$, $-0.014$, $-0.018$, $-0.016$ and the CKA to InfoNCE encoders of the other seeds is $0.962$, $0.967$, $0.971$, $0.975$, $0.977$, $0.975$ (InfoNCE seed to seed: $0.982$). The own-channel gains of these encoders lie in $[0.0017,0.0083]$. The switch-point encoders (InfoNCE after $4000$ steps; regenerated deterministically, with final branch weights identical to the reported runs) have style $R^2$ $-0.003,0.057,0.045$, $\J=(1.87,0.57,1.37)\times10^{-3}$ and $\J_{0.2}=(4.3,3.4,3.9)\times10^{-3}$ for the three seeds, evaluated with the same protocol.

\paragraph{Condition (U) in training (\cref{fig:utrain}, \cref{sec:linear}).} Moment matching minimizes $\|\E W\|^2+\|\E WW^\top-I/d\|_F^2+\alpha\Delta$ over the normalized embeddings $W$ of both views, with $\alpha=1$ (3 seeds); its encoders reach pair cosine $0.98$ and $\Ub=0.29$ (InfoNCE: $0.021$). The Gaussian uniformity loss is $\frac\beta2\E\|U-V\|^2+\log\mathrm{mean}_{i\ne j}e^{-\gamma\|w_i-w_j\|^2}$ over the $2B$ embeddings of a batch, with $\gamma\in\{1,2,2.5,5\}$ (2 seeds). The residual $\|\hat q_\ell\|^2$ is estimated from $3000$ test embeddings by the U-statistic $N_\ell\,\mathrm{mean}_{i\ne j}P_\ell(w_i^\top w_j)$ and averaged over seeds; for moment matching it is $7$--$63$ times InfoNCE's at every degree $\ell=3,\dots,8$. \emph{VICReg} with invariance weight $a\in\{1,2.5,5,10,25,50\}$ (variance and covariance weights $25,1$; 3 seeds, one at $a=50$) reaches pair cosine $0.91$--$0.99$ and $\Ub=0.9$--$2.0$ after normalization; style is not detected in any of the $16$ encoders.

\paragraph{Code and compute.} We will release the code, the training command of every encoder and all evaluated numbers; training is deterministic on CPU. Exact-law and controlled-latent experiments ran on a 2-core CPU; a latent-model run takes 2--13 minutes, and the heat-channel integration for $d=8$ about 40 minutes.

\begin{table*}[!t]
\caption{Content--style model, all metrics (mean $\pm$ s.d.\ over 3 seeds; \cref{tab:latent} reports a subset, together with the vMF CKA of \cref{tab:vmf}). Here $\cos$ is the mean pair cosine $\E\,U^\top V$; $\J$ in units of $10^{-3}$; excess is $\Lb-\mathcal L^\star_\beta=2\beta\Delta+\Ub$ at $\beta=5$; $R^2_c$ is the held-out linear decodability of the content; CKA and style as in \cref{tab:latent}. $\hat\J_U$ is the unbiased estimator; InfoNCE$\to$SG and InfoNCE$\to$SG$_{0.2}$ switch from InfoNCE halfway through training; $\lambda$ is LeJEPA's SIGReg weight.}\label{tab:latent-sd}
\centering\footnotesize
\setlength{\tabcolsep}{3pt}
\resizebox{\textwidth}{!}{\begin{tabular}{lccccccc}
\toprule
Method & $\cos$ & $\J$ & $\Ub$ & excess & $R^2_c$ & CKA & style\\
\midrule
InfoNCE & $0.985\pm0.001$ & $0.17\pm0.01$ & $0.02\pm{<}0.005$ & $0.10\pm{<}0.005$ & $0.77\pm0.01$ & $0.98\pm{<}0.005$ & $-0.02\pm0.01$\\
SG, $\hat\J_U$ & $0.802\pm0.006$ & $0.85\pm0.02$ & $0.05\pm0.01$ & $1.03\pm0.02$ & $0.57\pm0.04$ & $0.65\pm0.03$ & $0.11\pm0.01$\\
InfoNCE$\to$SG & $0.961\pm0.002$ & $0.12\pm0.01$ & $0.02\pm{<}0.005$ & $0.21\pm0.01$ & $0.75\pm0.01$ & $0.96\pm{<}0.005$ & $0.14\pm0.01$\\
InfoNCE$\to$SG$_{0.2}$ & $0.981\pm0.001$ & $0.27\pm0.01$ & $0.04\pm{<}0.005$ & $0.13\pm{<}0.005$ & $0.76\pm0.01$ & $0.98\pm{<}0.005$ & $-0.02\pm0.01$\\
\midrule
SG$_{0.2}$ & $0.973\pm0.001$ & $0.98\pm0.04$ & $0.16\pm{<}0.005$ & $0.29\pm0.01$ & $0.67\pm{<}0.005$ & $0.83\pm0.01$ & $-0.02\pm{<}0.005$\\
SG$_{0.05}$ & $0.934\pm0.002$ & $0.79\pm0.04$ & $0.12\pm{<}0.005$ & $0.45\pm0.01$ & $0.63\pm{<}0.005$ & $0.77\pm0.01$ & $-0.02\pm{<}0.005$\\
\midrule
LeJEPA, $\lambda=0.01$ & $0.956\pm{<}0.0005$ & $1.57\pm0.03$ & $0.26\pm0.01$ & $0.48\pm0.01$ & $0.56\pm{<}0.005$ & $0.70\pm{<}0.005$ & $-0.02\pm{<}0.005$\\
LeJEPA, $\lambda=0.1$ & $0.682\pm0.024$ & $2.20\pm0.36$ & $0.04\pm0.01$ & $1.63\pm0.12$ & $0.49\pm0.04$ & $0.57\pm0.03$ & $0.52\pm0.09$\\
VICReg & $0.979\pm0.002$ & $95\pm5$ & $1.38\pm0.06$ & $1.48\pm0.06$ & $0.47\pm0.03$ & $0.60\pm0.05$ & $-0.02\pm0.01$\\
\bottomrule
\end{tabular}
}
\vspace{1.2em}
\caption{vMF latent model ($d=8$, $B=256$, 8k steps; mean $\pm$ s.d. over 3 seeds). Columns as in \cref{tab:latent-sd}, without style; VICReg weights (sim, std, cov) $=(10,25,1)$.}\label{tab:vmf}
\centering\footnotesize
\resizebox{\textwidth}{!}{\begin{tabular}{lcccccc}
\toprule
Method & $\cos$ & $\J$ & $\Ub$ & excess & $R^2_c$ & CKA\\
\midrule
InfoNCE & $0.698\pm{<}0.0005$ & $1.37\pm0.01$ & $0.00\pm{<}0.005$ & $1.51\pm{<}0.005$ & $0.99\pm{<}0.005$ & $0.997\pm{<}0.0005$\\
SG, $\hat\J_U$ & $0.668\pm0.003$ & $1.75\pm0.08$ & $0.02\pm{<}0.005$ & $1.67\pm0.02$ & $0.93\pm0.01$ & $0.938\pm0.010$\\
SG, $\hat\J_V$ & $0.667\pm0.004$ & $1.76\pm0.08$ & $0.02\pm{<}0.005$ & $1.68\pm0.02$ & $0.92\pm0.01$ & $0.935\pm0.012$\\
\midrule
SG$_{0.2}$ & $0.694\pm0.001$ & $1.42\pm0.03$ & $0.00\pm{<}0.005$ & $1.53\pm{<}0.005$ & $0.99\pm{<}0.005$ & $0.993\pm0.001$\\
SG$_{0.05}$ & $0.686\pm0.001$ & $1.51\pm0.03$ & $0.01\pm{<}0.005$ & $1.57\pm{<}0.005$ & $0.97\pm{<}0.005$ & $0.981\pm0.001$\\
\midrule
LeJEPA, $\lambda=0.01$ & $0.687\pm0.001$ & $1.54\pm0.02$ & $0.01\pm{<}0.005$ & $1.57\pm0.01$ & $0.97\pm{<}0.005$ & $0.981\pm0.003$\\
VICReg, $(10,25,1)$ & $0.745\pm0.022$ & $57\pm18$ & $0.89\pm0.18$ & $2.16\pm0.07$ & $0.75\pm0.02$ & $0.811\pm0.009$\\
\bottomrule
\end{tabular}
}
\end{table*}

\clearpage
\section{Notation}\label{app:notation}

Symbols are listed with the place where they are defined; symbols used only in part of the paper are marked with the sections in which they appear.

\begin{center}
\footnotesize
\renewcommand{\arraystretch}{1.05}
\begin{tabular}{@{}>{\raggedright\arraybackslash}p{0.27\textwidth}>{\raggedright\arraybackslash}p{0.5\textwidth}>{\raggedright\arraybackslash}p{0.17\textwidth}@{}}
\toprule
Symbol & Meaning & Defined in\\
\midrule
\multicolumn{3}{@{}l}{\emph{Paired laws and the shared loss}}\\
$d$, $\Sd$, $\sigma$ & dimension, unit sphere, its uniform law & \cref{sec:setup}\\
$P$, $(U,V)$, $q_U,q_V$, $\bar q$, $q$ & paired law, views, marginals, mixture and equal marginal & \cref{sec:setup}\\
$\Delta(P)$ & misalignment $\frac14\E\|U-V\|^2$ & \eqref{eq:delta}\\
$C$, $R$, $Z$ & centroid, $\chi_d$ radius, $Z=RC$ & \eqref{eq:construction}\\
$\varphi_Z$, $\J(P)$ & characteristic function of $Z$, SG loss & \eqref{eq:J}\\
$a$, $t$ & frequency direction and scale & \eqref{eq:J}\\
$H$, $Q$ & $\frac12(U-V)$, $\|H\|^2$ & \cref{sec:alignment}\\
$j$, $k_0$, $\Delta^\star$ & alignment bound, $k_0=2\cdot3^{3/4}$, $\min\{1,j^{-1}\}$ & \cref{lem:align}, \cref{thm:infonce}\\
$\Dlift$, $K_{\rm lift}$ & lifted discrepancy and its zonal kernel & \cref{sec:infonce-comparison}, \cref{app:prelim}\\
$\J_\alpha$ (SG$_\alpha$) & $\J+\alpha\Delta$ & \eqref{eq:J-alpha}\\
\midrule
\multicolumn{3}{@{}l}{\emph{InfoNCE}}\\
$\beta=1/\tau$, $\Lb$, $\mathcal L^\star_\beta$ & inverse temperature, population InfoNCE, its minimum & \eqref{eq:infonce}\\
$\kappa_d(\beta)$, $\Ub(q)$ & $\E_{W\sim\sigma}e^{\beta u^\top W}$, uniformity excess & \eqref{eq:infonce-decomp}\\
$k_\beta$, $\Db$ & kernel $e^{\beta u^\top v}$ and its MMD & \cref{sec:infonce-comparison,sec:linear}\\
$C^\star_{d,\beta}$, $\bar c_{d,\beta}$, $c_{\beta,\alpha}$ & kernel-comparison and certificate constants & \cref{prop:kernel}, \cref{thm:infonce}, \cref{prop:calibrated}\\
\midrule
\multicolumn{3}{@{}l}{\emph{Finite batches}}\\
$B$, $M$, $\omega_m=t_ma_m$ & batch size, number of frequencies, frequencies & \cref{sec:finite}\\
$\hat\varphi_m$, $g_m$, $\hat\J_V$, $\hat\J_U$ & empirical and target characteristic functions, estimators & \cref{sec:finite}, \eqref{eq:J-U}\\
\midrule
\multicolumn{3}{@{}l}{\emph{Nuisance and price}}\\
$\vartheta$, $\J_q(\vartheta)$, $\J_{\rm unif}(q)$ & heat time, SG loss along the channel, its value at $\vartheta=0$ & \cref{sec:nuisance-channel}\\
$G_\ell$, $\rho(f)$, $\rho_\ell$, $\Delta_{\rm eq}$ & first-variation weights, equilibrium ratio ($\rho_\ell$: $f$ of pure degree $\ell$), equilibrium misalignment & \cref{thm:slaving}, \eqref{eq:balance}\\
$g$, $\alpha_c$ & gain of a channel, critical price & \eqref{eq:balance}, \cref{prop:condA}, \cref{cor:threshold}\\
(A) & alignment price above the gain, $\alpha>g$ & \cref{sec:alignment-price}, \cref{prop:condA}\\
$\lambda$, $S$, $I$, $g_S$, $\lambda_c$ & LeJEPA weight, SIGReg and invariance terms, gain, critical weight & \cref{prop:condA}\\
$W$, $\Lambda$, $c(\Lambda)$ & shared or mean code, latent, length of the conditional mean & \cref{sec:nuisance-channel,sec:alignment-price}\\
\midrule
\multicolumn{3}{@{}l}{\emph{Spherical harmonics and uniformity tests (from \cref{sec:linear} on)}}\\
$\hat q_\ell$ & degree-$\ell$ harmonic component of $q-\sigma$ & \cref{sec:linear}\\
$\mathcal H_\ell$, $N_\ell$, $Y_{\ell m}$, $P_\ell$ & harmonics of degree $\ell$, dimension, basis, Gegenbauer polynomial & \cref{app:spectral}\\
$\lambda^\beta_\ell$, $\lambda^{\rm lift}_\ell$, $\mu_\ell$, $S_K$ & eigenvalues of $k_\beta$, $K_{\rm lift}$, a zonal kernel $K$; spectral ratio & \cref{sec:linear}, \cref{thm:spectral}, \cref{app:spectral}\\
$\mathcal Q_\beta$, $\delta$, $c_K$, (U) & quadratic approximation of $\Ub$, density deviation, best linear constant, spectral dominance $S_K<\infty$ & \cref{thm:spectral}\\
$D=d/2$, $\nu=D-1$ & half-dimension, Bessel index & \cref{app:prelim}, \cref{app:spectral}\\
$g(t)=e^{-t^2/2}$, $T_\mu$, $\psi_u$ & target characteristic function (not the gain $g$), lifted variable, feature map & \cref{app:prelim}\\
$\mathsf k_d$, $K_d$, $\Gamma_q$ & centered-channel constant, gain constant, gain density & \cref{app:gain}\\
\bottomrule
\end{tabular}
\end{center}

\end{document}